\documentclass[letterpaper,11pt,twocolumn]{article}

\newcommand{\reportnumber}{TR-RoboPI-2026-01}
\newcommand{\reportdate}{October 23, 2026} % or \today
\newcommand{\shorttitle}{Adversarial Narrowband Acoustic Source Localization}
\newcommand{\fulltitle}{On Distributed Sensing and Adversarial Localization of Narrowband\\ Acoustic Sources for Offshore Data Center Monitoring}

\usepackage[letterpaper, top=1in, bottom=1in, left=0.75in, right=0.75in,
            headheight=14pt, headsep=14pt]{geometry}
\usepackage[T1]{fontenc}
\usepackage{lmodern}
\usepackage{microtype}

\usepackage{xspace,xcolor}
\definecolor{UFblue}{HTML}{0021A5}
\definecolor{UForange}{HTML}{FA4616}
\definecolor{lightgray}{gray}{0.95}

\usepackage{cite}
\usepackage{amsmath,amssymb,amsfonts}
\usepackage{amsthm}
\theoremstyle{plain}
\newtheorem{proposition}{Proposition}
\theoremstyle{remark}
\newtheorem{remark}{Remark}
\usepackage{graphicx}
\usepackage[size=small, labelfont={bf,sf}, labelsep=period]{caption}
\usepackage{subcaption}
\usepackage{booktabs}
\usepackage{tabularx}
\usepackage{multicol}
\usepackage{multirow}
\usepackage{soul}
\usepackage{siunitx}
\usepackage{physunits}
\usepackage{algorithm}
\usepackage{algorithmic}
\usepackage{array}
\usepackage[normalem]{ulem}
\usepackage{makecell}
\usepackage{wrapfig}
\usepackage{lipsum}
\usepackage{ragged2e}

\usepackage[shortlabels]{enumitem}

\usepackage{titlesec}
\titleformat{\section}
  {\normalfont\large\sffamily\bfseries\color{UFblue}}{\thesection}{0.6em}{}
\titleformat{\subsection}
  {\normalfont\normalsize\sffamily\bfseries\color{UFblue}}{\thesubsection}{0.5em}{}
\titleformat{\subsubsection}
  {\normalfont\normalsize\sffamily\itshape}{\thesubsubsection}{0.5em}{}
\titlespacing*{\section}{0pt}{2.0ex plus 0.5ex minus 0.2ex}{1.0ex plus 0.2ex}
\titlespacing*{\subsection}{0pt}{1.5ex plus 0.5ex minus 0.2ex}{0.6ex plus 0.2ex}
\titlespacing*{\subsubsection}{0pt}{1.2ex plus 0.3ex minus 0.2ex}{0.4ex}

\usepackage{fancyhdr}
\makeatletter
\let\NAT@parse\undefined
\makeatother
\usepackage[colorlinks=true, linkcolor=UFblue, citecolor=UFblue,
            urlcolor=UForange]{hyperref}
\DeclareMathOperator*{\argmin}{arg\,min}
\makeatletter
\DeclareRobustCommand\onedot{\futurelet\@let@token\@onedot}
\def\@onedot{\ifx\@let@token.\else.\null\fi\xspace}
\def\eg{\emph{e.g}\onedot} 
\def\ie{\emph{i.e}\onedot}

\makeatother
\long\def\invis#1{}

\newcommand{\codecomment}[1]{\textcolor{gray}{\textit{//#1}}}

\hypersetup{
  pdftitle={\fulltitle},
  pdfauthor={Adnan Abdullah, David Blow, Harsha Gundala Sureshbabu, Dongyoon Joung, Sara Rampazzi, Md Jahidul Islam},
  pdfsubject={\reportnumber}
}

\begin{document}

% TITLE PAGE (single column, no page number)

\begin{titlepage}
\centering
\sffamily

% --- Top banner ---
{\color{UFblue}\rule{\textwidth}{2pt}}\par\vspace{4pt}
\begin{minipage}{0.5\textwidth}
  \raggedright
  {\large\bfseries\color{UFblue} TECHNICAL REPORT}
\end{minipage}%
\begin{minipage}{0.5\textwidth}
  \raggedleft
  {\small\bfseries Report No.\ \reportnumber}\\[2pt]
  {\small \reportdate}
\end{minipage}\par
\vspace{4pt}{\color{UFblue}\rule{\textwidth}{0.6pt}}\par

\vspace{1.4in}

% --- Title ---
{\Large\bfseries\color{UFblue} \fulltitle \par}

\vspace{0.3in}
{\color{UForange}\rule{0.25\textwidth}{1.5pt}}\par
\vspace{0.3in}

% --- Authors ---
{\large
Adnan Abdullah$^{1}$, David Blow$^{1}$, Harsha Gundala Sureshbabu$^{1}$, Dongyoon Joung$^{1}$,\\ Sara Rampazzi$^{2}$, and Md Jahidul Islam$^{1}$\par}
\vspace{0.15in}
{\normalsize\color{gray}
$^{1}$RoboPI Laboratory, ECE Department, University of Florida\\
$^{2}$CPSec Laboratory, CISE Department, University of Florida\par}

\vspace{0.6in}

% --- Abstract box ---
\begin{minipage}{0.88\textwidth}
\colorbox{lightgray}{%
\begin{minipage}{\dimexpr\textwidth-2\fboxsep\relax}
\vspace{6pt}
\rmfamily\small
{\sffamily\bfseries\color{UFblue} Abstract}\par\vspace{4pt}
\justifying
This report investigates an acoustic surveillance framework for localizing and tracking narrowband adversarial acoustic sources that threaten offshore data centers. Unlike conventional sound source localization (SSL) methods that rely on broadband signals and compact receiver arrays, the considered scenario involves narrowband emissions and sparse sensors. In such scenarios, the observed time-difference-of-arrival (TDOA) measurements are aliased and produce an ambiguous solution space, rendering conventional localization methods ineffective. To address this challenge, we integrate a mobile receiver that temporarily forms anchor pairs with selected fixed receivers, constraining the feasible solution space. The proposed two-stage pipeline first estimates the initial source position using anchor constraints, then initializes a particle filter that jointly tracks source position and velocity using one live anchor measurement and quality-adaptive fixed-pair measurements. We evaluate the proposed method through Monte Carlo acoustic simulations, ROS-Gazebo robotic demonstrations, and open-water field experiments. Our framework achieves a localization success rate of $91\%$ in simulation, while $9$ of $10$ field trials show tracking error below $2$\,m under receiver-position uncertainty and noisy sensing conditions. These findings demonstrate the feasibility of passive acoustic surveillance for protecting large-scale infrastructure against acoustic threats.

\vspace{6pt}
{\sffamily\bfseries\color{UFblue} Keywords:}
Offshore surveillance; narrowband acoustic localization; data center monitoring.
\vspace{6pt}
\end{minipage}}
\end{minipage}

\vfill

% --- Footer of title page ---
{\color{UFblue}\rule{\textwidth}{0.6pt}}\par\vspace{4pt}
{\small\color{gray} University of Florida \,$\cdot$\, Gainesville, FL, USA}
\end{titlepage}

% MAIN BODY (two columns, numbering restarts at page 1)

\section{Introduction}

Modern distributed computing infrastructure, such as data centers, storage clusters, and edge computing facilities, relies on large numbers of spatially distributed hardware units~\cite{zhang2020distributed}. These systems are vulnerable to acoustic threats, where injected sounds can induce mechanical vibrations in sensitive components such as hard disk drives (HDDs)~\cite{bolton2018blue,sheldon2023deep}. Recent studies demonstrate that an adversarial agent can remotely inject malicious sound, leading to I/O disruptions and data unavailability in server systems~\cite{sheldon2024aquasonic}. Consequently, identifying the agent's location is a critical capability for monitoring and protecting distributed computing facilities~\cite{abdullah2025active}.

Acoustic source localization (ASL/SSL) has been extensively studied in both in-air and underwater environments using microphone/hydrophone arrays. These receiver arrays exploit spatially distributed measurements such as received signal strength (RSS)~\cite{chang2018target}, time difference of arrival (TDOA)~\cite{liu2023efficient,chen2017tdoa,mandic2019underwater}, frequency difference of arrival (FDOA)~\cite{jiang2020underwater}, and direction of arrival (DOA)~\cite{sun2023low} to estimate the acoustic source location. Most existing TDOA-based localization methods assume broadband signals, such as human speech, marine mammal vocalizations, pulsed signals, or environmental sounds, that provide unambiguous delay estimates between each receiver pair~\cite{hendricks2019automated,sun2017indoor}. In contrast, an adversarial source may intentionally emit a narrowband tonal signal near a resonance frequency to maximize damage to storage devices. Such periodicity induce \emph{phase aliasing} in TDOA measurements, where the observable delay is known only modulo the acoustic wavelength~\cite{reddy2014unambiguous,li2025enhanced}; see Fig.~\ref{fig:fig1}. Consequently, the receiver pairs yield finitely many feasible source locations, resulting in an ambiguous localization problem~\cite{ballal2010phase}.

\begin{figure}[t]
    \centering
    \includegraphics[width=\linewidth]{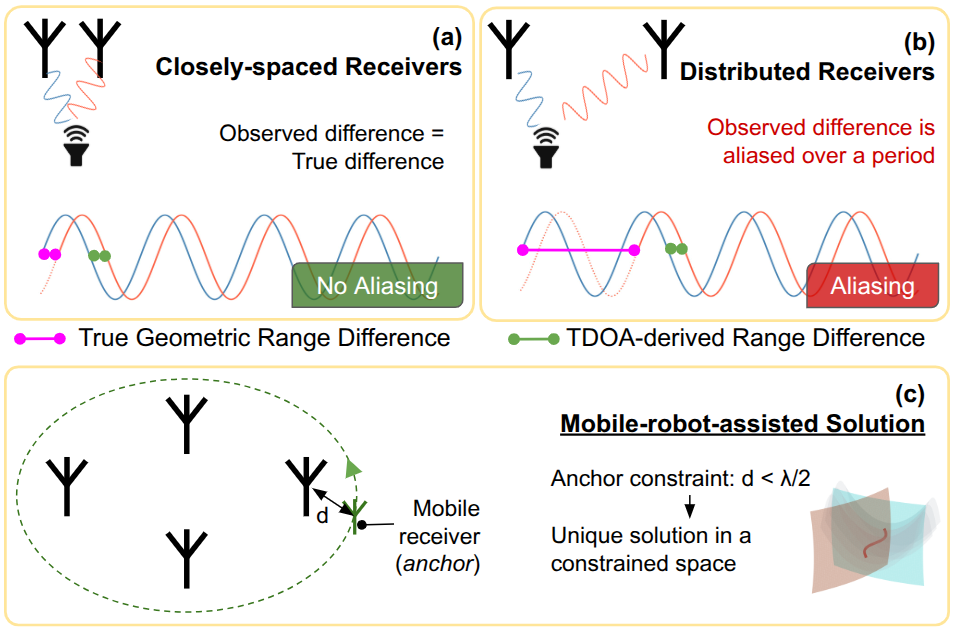}
    \vspace{-3mm}
    \caption{The narrowband phase aliasing problem and the proposed adversarial localization framework are illustrated. (a) Closely spaced receivers provide unaliased range differences that match the true geometric differences. (b) Distributed receivers observe phase-aliased differences, resulting in an ambiguous solution space. (c) A mobile robot forms short-baseline \textit{anchor pairs} with static receivers, providing geometric constraints for unique localization.}
    \label{fig:fig1}
    \vspace{-3mm}
\end{figure}

Practical surveillance requirements impose another challenge: acoustic sensing must be distributed across a large area. Deploying compact arrays at every infrastructure node is costly and difficult to scale. Conversely, a single centralized array is ineffective because signal attenuation and multipath propagation degrade measurements over long distances~\cite{ajdler2004acoustic}. Therefore, acoustic monitoring systems rely on sparsely distributed single receivers~\cite{azimi2019acoustic}. In such settings, the subset of receivers with reliable measurements varies with source location, necessitating quality-based adaptive channel weighting.

% The aforementioned challenges motivate a new formulation of adversarial acoustic source localization (\textbf{AASL}) that accounts for (i) sparse distributed sensing, (ii) narrowband phase aliasing, and (iii) receiver-dependent channel degradation. In this report, we address these challenges by introducing a mobile robot-assisted localization framework.
% Specifically, we augment the distributed receiver network with a mobile surveillance robot, \ie, an unmanned surface/underwater vehicle (USV/UUV) carrying an acoustic receiver. During an acoustic attack, the robot moves toward high-SNR receivers and forms temporary \emph{anchor pairs}. The TDOAs measured at anchor pairs provide necessary geometric constraints to collapse the ambiguity induced by phase aliasing. Using anchor constraints, we generate candidate source locations in the solution space, score them using alias-aware residuals from other receiver pairs, and refine the best candidates through local Gauss-Newton optimization to obtain the initial source position. In parallel, channel-quality scores are used to weight each measurement, reducing the influence of degraded channels. Finally, we initialize a particle filter around the refined estimate to jointly track the source position and velocity over time. 

These challenges motivate the investigation of adversarial acoustic source localization (\textbf{AASL}) under three coupled conditions: \textbf{(i)} sparse distributed sensing, \textbf{(ii)} narrowband phase aliasing, and \textbf{(iii)} receiver-dependent channel degradation. This report investigates the following research questions:

\begin{itemize}[leftmargin=*,nolistsep]
    \item \textbf{RQ1: Phase ambiguity.} Can unaliased geometric constraints obtained through mobile sensing resolve the phase ambiguity induced by narrowband signals and distributed receivers?

    \item \textbf{RQ2: Geometry and uncertainty.} How do anchor selection, receiver geometry, and measurement uncertainty affect initial localization accuracy and subsequent tracking performance?

    \item \textbf{RQ3: Channel degradation.} To what extent can quality-adaptive measurement weighting improve localization and tracking when individual receiver channels are degraded?

    \item \textbf{RQ4: Practical feasibility.} How effectively can mobile-assisted distributed acoustic surveillance be implemented across floating and submerged infrastructure and validated in open-water experiments?
\end{itemize}

To investigate these questions, we develop an adversarial localization framework that augments a distributed receiver network with a mobile acoustic sensor. The framework comprises two main technical components: \textbf{(i)} a mobile-assisted sensing architecture that forms temporary short-baseline \emph{anchor pairs} with selected static receivers to obtain unaliased geometric constraints; and \textbf{(ii)} a two-stage localization and tracking pipeline that combines anchor-constrained initialization, alias-aware candidate refinement, particle-filter tracking, and quality-adaptive measurement weighting. During initialization, the mobile receiver visits selected static receivers to acquire two independent anchor constraints while the source is assumed to remain approximately stationary. These constraints reduce the feasible solution space to a one-dimensional curve under suitable receiver geometry. Candidate locations along this curve are evaluated using phase-aliased measurements from other receiver pairs and refined through local Gauss-Newton optimization. The resulting estimate initializes a particle filter that jointly tracks source position and velocity, while channel-quality weighting reduces the influence of unreliable measurements.

Our investigation spans three evaluation settings. First, Monte Carlo acoustic simulations examine phase ambiguity, receiver geometry, measurement uncertainty, and channel degradation under controlled conditions. Second, ROS-Gazebo robotic simulations evaluate the integration of mobile sensing and source tracking in floating-barge and submerged-pod surveillance scenarios. Third, open-water experiments assess the framework using a scaled floating data center prototype equipped with six distributed receivers, a mobile receiver platform, and a USV-mounted adversarial source. Performance is evaluated using initial localization error, localization success rate, and temporal tracking error.

The results demonstrate the importance of mobile anchor constraints for resolving narrowband localization ambiguity. In Monte Carlo simulations, the two-anchor configuration achieves a $91\%$ initial localization success rate, outperforming under-constrained configurations. Sensitivity analyses further reveal that localization reliability depends strongly on anchor measurement accuracy, while quality-adaptive weighting improves robustness under receiver-channel degradation. Robotic simulations demonstrate tracking across different infrastructure configurations, and open-water experiments achieve successful localization with tracking error below $2$\,m in $9$ of $10$ trials. Together, our investigations characterize both the potential of distributed acoustic surveillance and the practical importance of receiver geometry and measurement reliability in offshore deployments.

\section{Background \& Preliminaries}
We consider a passive surveillance system consisting of $N$ spatially distributed acoustic receivers (microphones/hydrophones) at known positions $\{\mathbf{r}_i\}_{i=1}^{N}$ and an additional receiver mounted on a rail or a mobile robot. The mobile receiver can move within the environment and approach a static receiver to form a temporary \emph{anchor} pair. In contrast to conventional receiver arrays, this static-mobile configuration enables scalable coverage with reduced sensors. The objective is to estimate the state of an adversarial acoustic source $\mathbf{x} = [\mathbf{p}^\top, \mathbf{v}^\top]^\top \in \mathbb{R}^{2d}$, where $\mathbf{p}$ and $\mathbf{v}$ denote its position and velocity, respectively, and $d \in \{2,3\}$ denotes the spatial dimension of the environment. This section introduces the adversary model, derives the narrowband TDOA measurement model, and summarizes the limitations of existing localization techniques that motivate this work.

\subsection{Acoustic Injection Attacks in Data Centers}
The monitored data center facility contains sound-sensitive storage devices (\eg, HDDs), housed within one or more pressure-sealed enclosures referred to as pods. A remote adversary can emit a tonal signal near the structural resonant frequency, inducing mechanical vibrations that propagate through the enclosure and disturb the spinning platters and read/write actuator of an HDD~\cite{bolton2018blue,sheldon2023deep,sheldon2024aquasonic}. Such disturbances can cause I/O failures, data corruption, application crashes, and denial of service in computer systems. This acoustic vulnerability is particularly relevant to offshore and underwater data centers, where the surrounding water and the internal nitrogen-filled chamber provide stronger acoustic coupling~\cite{blow2025detection}. 
We model the adversarial acoustic source as a stationary or slowly moving, non-cooperative agent emitting a narrowband sound with dominant frequency $f_0$. This narrowband signaling condition leads to phase aliasing and affects TDOA estimation in distributed single-receiver systems.

\subsection{TDOA Measurement and Phase Aliasing}
\label{subsec:tdoa_aliasing}
TDOA-based localization relates the difference in signal arrival times across two receivers to their geometric range difference. For a receiver pair $(i,j)$ with known positions $\mathbf{r}_i$ and $\mathbf{r}_j$, the true source-to-receiver range difference is
\begin{equation}
\Delta \rho_{ij}(\mathbf{p}) = \|\mathbf{p} - \mathbf{r}_i\| - \|\mathbf{p} - \mathbf{r}_j\|.
\label{eq:range_difference}
\end{equation}
The locus of positions satisfying \eqref{eq:range_difference} forms a hyperbola in two dimensions and a hyperboloid manifold in three dimensions. Measurements from a sufficient number of independent pairs yield a unique intersection of these manifolds, corresponding to the estimated source location~\cite{diez20193d,park2018spectral}.

For broadband acoustic signals, TDOA is commonly estimated using generalized cross-correlation methods such as GCC-ML and GCC-PHAT~\cite{knapp2003generalized,chen2006time}. These methods exploit phase information from multiple frequency bins to identify a dominant correlation peak. For a single-frequency signal, however, the correlation is periodic and does not provide a unique TDOA~\cite{salvati2021time,yang2025neural}. We instead extract the complex phasor at the dominant frequency. For a signal window $x_i[n]$ measured at receiver $i$,
\begin{equation}
    A_i
    =
    \sum_{n}
    w[n]\,x_i[n]\,
    e^{-j2\pi f_0 n/f_s},
    \label{eq:phasor_extraction}
\end{equation}
where $w[n]$ is a Hann window and $f_s$ is the sampling frequency.
The relative phase between receivers $i$ and $j$ is $\phi_{ij} = \angle\left(A_iA_j^{*}\right)$,
which yields the TDOA in range-difference units (meters):
\begin{equation}
\hat{z}_{ij}
=
-\frac{\lambda}{2\pi}\phi_{ij},
\label{eq:phase_to_range}
\end{equation}
where $\lambda={c}/{f_0}$ is the signal wavelength.

\begin{figure*}[t]
    \centering
    \includegraphics[width=\linewidth]{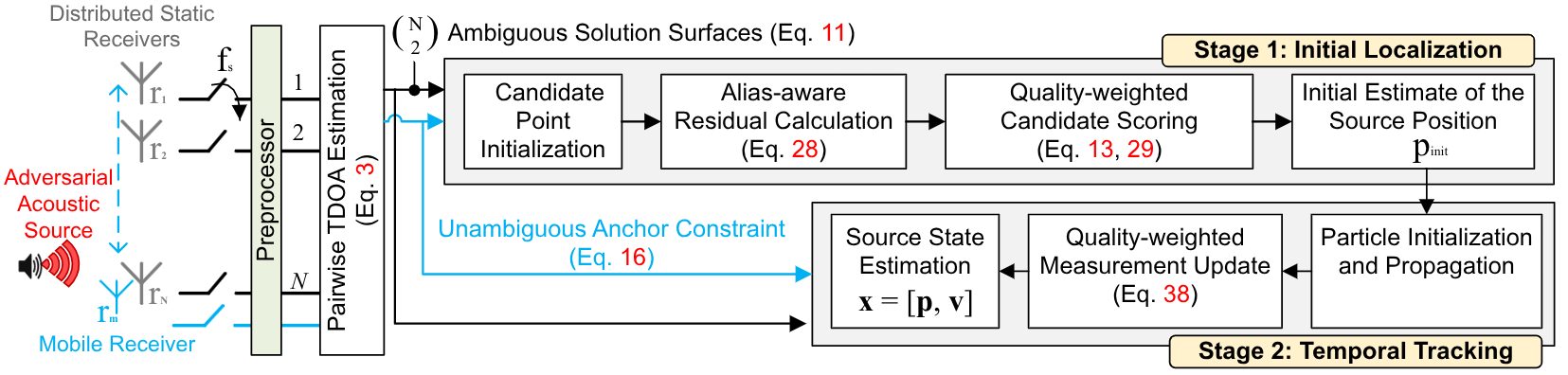}
    \caption{The proposed AASL framework utilizes $N$ spatially-distributed static receivers and a mobile receiver that forms temporary anchor pairs with static receivers. Each channel senses the narrowband signal emitted by the acoustic source, which is processed to estimate pairwise TDOAs. In Phase 1, anchor constraints from the mobile receiver are combined with ambiguous fixed-pair measurements to generate, score, and refine candidate source locations in the solution space, yielding the initial position estimate $\hat{\mathbf{p}}_{\mathrm{init}}$. In Phase 2, a particle filter jointly tracks the source position and velocity over time using one live anchor measurement and quality-weighted fixed-pair measurements.}
    \label{fig:pipeline}
\end{figure*}

\vspace{1mm}
\noindent
\textbf{Phase aliasing.} Let $b_{ij} = \|\mathbf{r}_i-\mathbf{r}_j\|$ denote the baseline, \ie, the receiver separation. By the reverse triangle inequality, every physically possible range difference is bounded by the baseline:
\begin{equation}
    |\Delta\rho_{ij}(\mathbf{p})|
    \leq
    b_{ij}.
    \label{eq:range_bound}
\end{equation}
When $b_{ij}\geq\lambda/2$, as is typical for sparsely distributed receivers, different range differences separated by integer multiples of $\lambda$ may produce the same observed phase~\cite{ballal2010phase}. Hence, the observed difference in~\eqref{eq:phase_to_range} is essentially the true range difference of~\eqref{eq:range_difference} modulo $\lambda$:
\begin{equation}
    \hat{z}_{ij}
    =
    \mathrm{modulo}_{\lambda}
    \left(
        \Delta\rho_{ij}(\mathbf{p})
    \right),
    \label{eq:wrapped_tdoa}
\end{equation}
where $\mathrm{modulo}_\lambda(\cdot)$ maps its argument to $[-\lambda/2,\, \lambda/2)$.
Equivalently, there exists an unknown integer alias index $k_{ij}\in\mathbb{Z}$ such that
\begin{equation}
\Delta\rho_{ij}(\mathbf{p})
=
\hat{z}_{ij}
+
k_{ij}\lambda.
\label{eq:alias_index}
\end{equation}
Although $k_{ij}$ is unknown, the physical bound in~\eqref{eq:range_bound} restricts it to a finite set, as formalized below.

\begin{proposition}[Finite feasible alias set]
\label{prop:finite_alias_set}
For a receiver pair with baseline $b_{ij}$ and wrapped measurement $\hat{z}_{ij}$, the physically feasible alias indices belong to the finite set
\begin{equation}
    \mathcal{K}_{ij}
    =
    \left\{
    k\in\mathbb{Z}:
    k_{ij}^{\min}
    \leq k
    \leq
    k_{ij}^{\max}
    \right\},
    \label{eq:feasible_alias_set}
\end{equation}
where
\begin{equation}
    k_{ij}^{\min}
    =
    \left\lceil
    \frac{-b_{ij}-\hat{z}_{ij}}{\lambda}
    \right\rceil,
    \qquad
    k_{ij}^{\max}
    =
    \left\lfloor
    \frac{b_{ij}-\hat{z}_{ij}}{\lambda}
    \right\rfloor.
    \label{eq:k_minmax}
\end{equation}
The number of physically feasible alias branches is therefore
\begin{equation}
    N_{ij}
    =
    \max\left(
    0,\,
    k_{ij}^{\max}-k_{ij}^{\min}+1
    \right).
    \label{eq:num_aliases}
\end{equation}
\end{proposition}

\begin{proof}
Combining the physical bound in~\eqref{eq:range_bound} with the alias relation in~\eqref{eq:alias_index} gives
\begin{equation}
    -b_{ij}
    \leq
    \hat{z}_{ij}+k_{ij}\lambda
    \leq
    b_{ij}.
    \label{eq:alias_bound}
\end{equation}
Subtracting $\hat{z}_{ij}$ and dividing by the positive wavelength $\lambda$ yields
\begin{equation*}
    \frac{-b_{ij}-\hat{z}_{ij}}{\lambda}
    \leq
    k_{ij}
    \leq
    \frac{b_{ij}-\hat{z}_{ij}}{\lambda}.
\end{equation*}
Since $k_{ij}$ must be an integer, the smallest and largest feasible values are given by~\eqref{eq:k_minmax}. Thus, the feasible alias set is finite and contains the number of indices stated in~\eqref{eq:num_aliases}.
\end{proof}

For each $k\in\mathcal{K}_{ij}$, the level set
\begin{equation}
    \mathcal{H}_{ij}^{(k)}
    =
    \left\{
    \mathbf{p}\in\mathbb{R}^{d}:
    \Delta\rho_{ij}(\mathbf{p})
    =
    \hat{z}_{ij}+k\lambda
    \right\}
    \label{eq:alias_branch}
\end{equation}
defines one feasible hyperbola branch in two dimensions or one hyperboloid sheet in three dimensions, except at degenerate limiting values. Thus, the observed range difference from each distributed receiver pair is associated with a family of feasible branches instead of a unique one. Although the number of branches is finite, each branch contains a continuous set of source positions; consequently, a single aliased pair is insufficient for unique localization. Since $N_{ij}$ scales approximately with $2b_{ij}/\lambda$, widely separated receivers or higher-frequency signals produce more severe spatial aliasing.

\begin{remark}[Unaliased condition]
\label{rem:unaliased}
If $b_{ij}<\lambda/2$, the physically feasible range $[-b_{ij},b_{ij}]$ lies strictly inside the principal phase interval $[-\lambda/2,\lambda/2)$. In this case, $\mathcal{K}_{ij}=\{0\}$, and the measurement $\hat{z}_{ij}$ equals the true range difference, \ie, it is unaliased.
\end{remark}

\vspace{1mm}
\noindent
\textbf{Multipath model.} To account for the shallow-water multipath effect, we adopt the first-order image-source reflection model from~\cite{abdullah2025active}. For floating scenarios, we include the water-surface reflection. For submerged scenarios, we include both water-surface and bed reflections. The multipath-affected range difference is modeled as
\begin{equation}
    \Delta\rho_{ij}^{(\mathrm{mp})}
    =
    \Delta\rho_{ij}
    +
    \alpha_{\mathrm{s}}\,
    \delta_{ij}^{(\mathrm{s})}
    +
    \alpha_{\mathrm{b}}\,
    \delta_{ij}^{(\mathrm{b})},
    \label{eq:multipath}
\end{equation}
where $\delta_{ij}^{(\mathrm{s})}$ and $\delta_{ij}^{(\mathrm{b})}$ are the range differences associated with the surface and bed image sources, respectively, and $\alpha_{\mathrm{s}}$ and $\alpha_{\mathrm{b}}$ are the corresponding reflection coefficients.

\subsection{Prior Works \& Motivation for Mobile Sensing}

Large-area SSL is typically performed by networking multiple local arrays and fusing their direction or location estimates through distributed optimization, regression, or graph-based models~\cite{albertini2025diffusion,hahmann2022sound,feng2023soft,lai2024two,ajdler2004acoustic}. Learning-based methods further employ spectral and spatial features extracted from multichannel array observations~\cite{yan2023indoor,liu2025sound,laufer2017semi,grinstein2023graph}. However, these approaches generally assume a compact array at each sensing location and signals with sufficient bandwidth. Deploying such arrays at every node of a large offshore facility substantially increases sensor count, synchronization complexity, cost, and maintenance requirements~\cite{abdullah2025active}. Distributed single receivers provide a scalable alternative, but their large separation relative to $\lambda$ produces severe spatial aliasing for tonal signals. Moreover, the received signal quality varies with the source location, attenuation, and propagation conditions.

These limitations motivate the integration of a mobile sensor in a distributed sensing architecture. The mobile robot utilizes static receivers with high SNRs to form anchor pairs that provide necessary geometric constraints to restrict the feasible solutions. Measurements from the remaining distributed receiver pairs are then incorporated through alias-aware residuals and channel-quality weights, enabling scalable localization and tracking of adversarial acoustic sources.

\section{Localization and Tracking Method}
\label{sec:method}
The proposed framework operates in two stages. First, a mobile receiver forms short-baseline anchor pairs with selected static receivers to estimate an initial source position under phase-aliased conditions. Second, a particle filter is initialized around this estimate and jointly tracks the source position and velocity using one live anchor measurement and quality-weighted measurements from the distributed receiver network. Both stages rely on a quality-adaptive weighting of the distributed receiver pairs, which we introduce first. Fig.~\ref{fig:pipeline} shows an overview of the framework.

\subsection{Quality-Adaptive Pair Weighting}
\label{subsec:adaptive_weighting}
Distributed receivers experience different signal qualities due to propagation loss, orientation, occlusion, and local interference. A degraded channel affects every TDOA pair containing that receiver; therefore, assigning equal confidence to all pairs can allow one poor channel to corrupt a large portion of the measurement set.
To address this, we estimate the received signal quality at each channel and prioritize the channels accordingly. For each receiver $i$, a narrowband SNR $\widehat{\mathrm{SNR}}_i$ is estimated around the dominant source frequency $f_0$. The receiver quality is mapped to $[q_{\min},1]$ using
\begin{equation}
    q_i = q_{\min} +
    \frac{1-q_{\min}}
    {1+ \exp\left[-\alpha
    \left(\widehat{\mathrm{SNR}}_i-\mathrm{SNR}_0\right)
    \right]
    },
    \label{eq:receiver_quality}
\end{equation}
where $\alpha$ controls the transition slope and $\mathrm{SNR}_0$ is the midpoint of the quality mapping.
We define the reliability of receiver pair $(i,j)$ as $q_{ij}=\min(q_i,q_j)$, reflecting that the quality of a pair is limited by its poorer constituent channel. These weights are recomputed for every audio frame and are used in both initial candidate scoring (Sec.~\ref{subsec:candidate_scoring}) and temporal tracking (Sec.~\ref{sec:temporal_tracking}). Anchor measurements retain unit weight, since they provide the unaliased geometric constraints required by the estimator.

\subsection{Stage 1: Initial Localization}
\label{subsec:initial_localization}

The goal of Stage 1 is to resolve the alias ambiguity described in Sec.~\ref{subsec:tdoa_aliasing} and obtain an initial source-position estimate $\hat{\mathbf{p}}_{\mathrm{init}}$. During this stage, the source is assumed stationary while the robot sequentially visits the selected static receivers to form the required anchor pairs. It proceeds in three steps: anchor measurements reduce the feasible solution space to a 1D curve (Sec.~\ref{subsec:anchor_constraint}), this curve is traced and sampled (Sec.~\ref{subsec:curve_tracing}), and the samples are scored and refined using the aliased fixed-pair measurements (Sec.~\ref{subsec:candidate_scoring}).

\subsubsection{Anchor Constrained Solution Space}
\label{subsec:anchor_constraint}
Let $A$ anchor pairs be available in a $d$-dimensional environment. For anchor pair $a\in\{1,\dots,A\}$, let $\mathbf{r}_a$ denote the position of the static anchor receiver and $\mathbf{r}_{m,a}$ the position of the co-located mobile receiver. The anchor baseline is $b_a = \|\mathbf{r}_{m,a}-\mathbf{r}_a\|$, and the anchor range difference is
\begin{equation}
    \Delta\rho_a(\mathbf{p})
    =
    \|\mathbf{p}-\mathbf{r}_a\|
    -
    \|\mathbf{p}-\mathbf{r}_{m,a}\|.
    \label{eq:anchor_range_difference}
\end{equation}
When $b_a<\lambda/2$, Remark~\ref{rem:unaliased} guarantees that $|\Delta\rho_a(\mathbf{p})|\leq b_a<\lambda/2$ for any source position, and the anchor measurement is therefore unaliased. We model it as
\begin{equation}
    \hat{d}_a
    =
    \Delta\rho_a(\mathbf{p})+\eta_a,
    \qquad
    \eta_a\sim\mathcal{N}(0,\sigma_a^2),
    \label{eq:anchor_meas}
\end{equation}
where $\eta_a$ captures acoustic measurement error and uncertainty in the mobile-receiver pose. Each measurement constrains the source to the manifold
\begin{equation}
    \mathcal{S}_a
    =
    \left\{
    \mathbf{p}\in\mathbb{R}^{d}:
    \Delta\rho_a(\mathbf{p})=\hat{d}_a
    \right\}.
    \label{eq:anchor_surface}
\end{equation}
We keep the anchor baseline close to, but below, $\lambda/2$ to retain unaliased measurements while preserving geometric diversity in the TDOA observations.

\vspace{1mm}
\noindent
\textbf{Anchor-consistent solution set.}
To characterize the set of source positions consistent with all anchor measurements, we define the residual function of anchor pair $a$ as
\begin{equation}
    f_a(\mathbf{p})
    =
    \Delta\rho_a(\mathbf{p})
    -
    \hat{d}_a,
    \label{eq:anchor_constraint}
\end{equation}
and stack all anchor residuals into the constraint vector
\begin{equation}
    \mathbf{f}(\mathbf{p})
    =
    \begin{bmatrix}
        f_1(\mathbf{p}) &
        \cdots &
        f_A(\mathbf{p})
    \end{bmatrix}^{\top}.
    \label{eq:constraint_vector}
\end{equation}
The anchor-consistent set is the intersection of the individual anchor manifolds:
\begin{equation}
    \mathcal{C}_A
    =
    \left\{
    \mathbf{p}\in\mathbb{R}^{d}:
    \mathbf{f}(\mathbf{p})=\mathbf{0}
    \right\}
    =
    \bigcap_{a=1}^{A}\mathcal{S}_a.
    \label{eq:anchor_set}
\end{equation}
The gradient of a single anchor constraint is the difference of two unit vectors pointing from the receivers toward the source,
\begin{equation}
    \nabla_{\mathbf{p}}f_a(\mathbf{p})
    =
    \frac{\mathbf{p}-\mathbf{r}_a}
    {\|\mathbf{p}-\mathbf{r}_a\|}
    -
    \frac{\mathbf{p}-\mathbf{r}_{m,a}}
    {\|\mathbf{p}-\mathbf{r}_{m,a}\|},
    \label{eq:anchor_gradient}
\end{equation}
and stacking these row vectors yields the anchor Jacobian
\begin{equation}
    \mathbf{J}(\mathbf{p})
    =
    \nabla_{\mathbf{p}}\mathbf{f}(\mathbf{p})
    \in
    \mathbb{R}^{A\times d}.
    \label{eq:anchor_jacobian}
\end{equation}

\begin{proposition}[Required number of independent anchors]
\label{prop:anchor_dimension}
Suppose $\mathbf{p}^{\star}$ satisfies $\mathbf{f}(\mathbf{p}^{\star})=\mathbf{0}$ and the anchor Jacobian has full row rank,
\begin{equation}
    \operatorname{rank}
    \left(
    \mathbf{J}(\mathbf{p}^{\star})
    \right)
    =
    A.
    \label{eq:full_row_rank}
\end{equation}
Then, in a neighborhood of $\mathbf{p}^{\star}$, the anchor-consistent set $\mathcal{C}_A$ is a smooth manifold of dimension $\dim(\mathcal{C}_A) = d-A$. Therefore, $A=d-1$ independent anchor constraints are required to reduce the feasible source set to a one-dimensional candidate manifold.
\end{proposition}

\begin{proof}
Under the full-row-rank condition in~\eqref{eq:full_row_rank}, the implicit function theorem applies to $\mathbf{f}:\mathbb{R}^{d}\rightarrow\mathbb{R}^{A}$. Its zero-level set is locally a smooth manifold whose dimension equals the ambient dimension minus the rank of the constraint Jacobian. Hence,
\begin{equation*}
    \dim(\mathcal{C}_A)
    =
    d-\operatorname{rank}(\mathbf{J})
    =
    d-A.
\end{equation*}
Setting $\dim(\mathcal{C}_A)=1$ gives $A=d-1$.
\end{proof}

Proposition~\ref{prop:anchor_dimension} yields the anchor requirements used throughout this report:
\begin{itemize}[leftmargin=*, nosep]
    \item In 2D, a single anchor constraint ($A=1$) yields a one-dimensional hyperbola branch, $\mathcal{C}=\mathcal{S}_1$.
    \item In 3D, two independent anchor constraints ($A=2$) intersect along a one-dimensional spatial curve, $\mathcal{C}=\mathcal{S}_1\cap\mathcal{S}_2$.
\end{itemize}
We hereafter use $\mathcal{C}=\mathcal{C}_{d-1}$ for this candidate curve. The anchor constraints alone cannot determine where the source lies along $\mathcal{C}$. The remaining aliased fixed-pair measurements are used to resolve this, after the curve is traced and sampled.

\subsubsection{Intersection-Curve Tracing}
\label{subsec:curve_tracing}
Before the fixed-pair measurements can be used to locate the source on $\mathcal{C}$, the curve must be represented by a discrete set of points. Since $\mathcal{C}$ is defined only implicitly by $\mathbf{f}(\mathbf{p})=\mathbf{0}$, we obtain these points through a seed search followed by predictor-corrector continuation along the curve. The procedure below is stated for the general case of $A=d-1$ anchor constraints, so that $\mathbf{f}:\mathbb{R}^{d}\rightarrow\mathbb{R}^{d-1}$ and $\mathbf{J}(\mathbf{p})\in\mathbb{R}^{(d-1)\times d}$.

\vspace{1mm}
\noindent
\textbf{Seed identification and Newton projection.}
We first draw $M$ candidate points $\{\mathbf{p}_m\}_{m=1}^{M}$ uniformly from the monitored region $\mathcal{B}\subset\mathbb{R}^{d}$ and select the one that best satisfies the anchor constraints:
\begin{equation}
    \mathbf{p}_{\mathrm{seed}}
    =
    \argmin_{\mathbf{p}_m}
    \left\|
    \mathbf{f}(\mathbf{p}_m)
    \right\|_2^2.
    \label{eq:seed_search}
\end{equation}
The sampled seed generally lies near, but not exactly on, the curve. Linearizing the constraints around the current point gives
\begin{equation}
    \mathbf{f}(\mathbf{p}+\delta\mathbf{p})
    \approx
    \mathbf{f}(\mathbf{p})
    +
    \mathbf{J}(\mathbf{p})\,\delta\mathbf{p},
    \label{eq:constraint_linearization}
\end{equation}
and the smallest correction that drives this linearized residual to zero is the damped minimum-norm step
\begin{equation}
    \delta\mathbf{p}
    =
    -
    \mathbf{J}^{\top}
    \left(
    \mathbf{J}\mathbf{J}^{\top}
    +
    \mu\mathbf{I}
    \right)^{-1}
    \mathbf{f}(\mathbf{p}),
    \label{eq:newton_projection}
\end{equation}
where $\mu>0$ regularizes nearly singular anchor geometries, \ie, when the anchor gradients in~\eqref{eq:anchor_gradient} are close to parallel. Repeating $\mathbf{p}\leftarrow\mathbf{p}+\delta\mathbf{p}$ until $\|\mathbf{f}(\mathbf{p})\|_2<\epsilon_f$ projects the point onto $\mathcal{C}$, where $\epsilon_f$ is the projection tolerance. Because the step has minimum norm, the projection moves the point orthogonally toward the curve rather than sliding it along the curve.

\vspace{1mm}
\noindent
\textbf{Predictor-corrector tracing.}
At a regular point on $\mathcal{C}$, the tangent direction $\mathbf{t}\in\mathbb{R}^{d}$ satisfies $\mathbf{J}(\mathbf{p})\,\mathbf{t}=\mathbf{0}$. Under the full-rank condition of Proposition~\ref{prop:anchor_dimension}, this null space is one-dimensional, and the unit tangent is the right singular vector of $\mathbf{J}$ associated with its smallest singular value. In 3D, this is equivalent to
\begin{equation}
    \mathbf{t}
    \propto
    \nabla f_1(\mathbf{p})
    \times
    \nabla f_2(\mathbf{p}),
    \label{eq:tangent_cross_product}
\end{equation}
while in 2D the tangent is the single anchor gradient rotated by $90^{\circ}$, \ie, $\mathbf{t}\propto[-\partial_y f_1,\ \partial_x f_1]^{\top}$.
Given the current curve point $\mathbf{p}_n$, a predictor step advances along the tangent,
\begin{equation}
    \widetilde{\mathbf{p}}_{n+1}
    =
    \mathbf{p}_n
    +
    s\,h\,\mathbf{t}_n,
    \qquad s\in\{-1,+1\},
    \label{eq:predictor}
\end{equation}
where $h$ is the tracing step size and $s$ selects the tracing direction. The corrector then projects $\widetilde{\mathbf{p}}_{n+1}$ back onto $\mathcal{C}$ using~\eqref{eq:newton_projection}, yielding $\mathbf{p}_{n+1}$. Because the singular-vector computation determines $\mathbf{t}_n$ only up to sign, we enforce $\mathbf{t}_n^{\top}\mathbf{t}_{n-1}\geq 0$ by flipping $\mathbf{t}_n$ when necessary, which prevents the trace from reversing direction.

We trace the curve in both directions from $\mathbf{p}_{\mathrm{seed}}$ and terminate each direction when the corrected point leaves $\mathcal{B}$, when the projection fails to reach $\epsilon_f$, or after $L_{\max}$ steps. The resulting set of curve samples, denoted $\mathcal{P}=\{\mathbf{p}_n\}$, is a discrete representation of $\mathcal{C}$ within the monitored region. This reduces the search from the full $d$-dimensional region to a one-dimensional set of candidates.

\subsubsection{Candidate Scoring and Alias-Aware Refinement}
\label{subsec:candidate_scoring}
Every sample in $\mathcal{P}$ satisfies the anchor constraints within the projection tolerance $\epsilon_f$. Hence, the task is to determine which sample is most consistent with the aliased fixed-receiver measurements.

\begin{algorithm}[!h]
\caption{Stage 1: Intersection-Curve Tracing and Alias-Aware Refinement}
\label{alg:curve_refinement}
\small
\begin{algorithmic}[1]
\REQUIRE Anchor constraints $\mathbf{f}(\mathbf{p})$; fixed-pair measurements $\{\hat{z}_{ij}\}_{(i,j)\in\mathcal{M}}$ and weights $\{q_{ij}\}$; wavelength $\lambda$; region $\mathcal{B}$; parameters $M, h, L_{\max}, \epsilon_f, K, I_{\max}$
\ENSURE Initial source-position estimate $\hat{\mathbf{p}}_{\mathrm{init}}$
\STATE \codecomment{ Seed identification}
\STATE Draw $M$ points $\{\mathbf{p}_m\}_{m=1}^{M}$ uniformly from $\mathcal{B}$
\STATE $\mathbf{p}_{\mathrm{seed}}\leftarrow\argmin_{\mathbf{p}_m}\|\mathbf{f}(\mathbf{p}_m)\|_2^2$ \hfill\eqref{eq:seed_search}
\STATE Project $\mathbf{p}_{\mathrm{seed}}$ onto $\mathcal{C}$ via~\eqref{eq:newton_projection} until $\|\mathbf{f}\|_2<\epsilon_f$
\STATE $\mathcal{P}\leftarrow\{\mathbf{p}_{\mathrm{seed}}\}$
\STATE \codecomment{ Predictor-corrector tracing}
\FOR{each direction $s\in\{-1,+1\}$}
    \STATE $\mathbf{p}\leftarrow\mathbf{p}_{\mathrm{seed}}$, \ $\mathbf{t}_{\mathrm{prev}}\leftarrow\emptyset$
    \FOR{$n=1$ to $L_{\max}$}
        \STATE $\mathbf{t}\leftarrow$ unit null-space vector of $\mathbf{J}(\mathbf{p})$
        \IF{$\mathbf{t}_{\mathrm{prev}}\neq\emptyset$ \AND $\mathbf{t}^{\top}\mathbf{t}_{\mathrm{prev}}<0$}
            \STATE $\mathbf{t}\leftarrow-\mathbf{t}$
        \ENDIF
        \STATE $\widetilde{\mathbf{p}}\leftarrow\mathbf{p}+s\,h\,\mathbf{t}$ \hfill\eqref{eq:predictor}
        \STATE Project $\widetilde{\mathbf{p}}$ onto $\mathcal{C}$ via~\eqref{eq:newton_projection}
        \IF{projection fails \OR $\widetilde{\mathbf{p}}\notin\mathcal{B}$}
            \STATE \textbf{break}
        \ENDIF
        \STATE $\mathcal{P}\leftarrow\mathcal{P}\cup\{\widetilde{\mathbf{p}}\}$, \ $\mathbf{p}\leftarrow\widetilde{\mathbf{p}}$, \ $\mathbf{t}_{\mathrm{prev}}\leftarrow\mathbf{t}$
    \ENDFOR
\ENDFOR
\STATE \codecomment{ Candidate scoring}
\STATE Compute $S(\mathbf{p}_n)$ for all $\mathbf{p}_n\in\mathcal{P}$ \hfill\eqref{eq:candidate_score}
\STATE Keep the top-$K$ samples as candidates $\{\mathbf{p}^{(0)}_k\}_{k=1}^{K}$
\STATE \codecomment{ Alias-aware refinement}
\FOR{$k=1$ to $K$}
    \STATE $\mathbf{p}\leftarrow\mathbf{p}^{(0)}_k$
    \FOR{$\ell=1$ to $I_{\max}$}
        \STATE Compute $k_{ij}^{\star}(\mathbf{p})$ and $r_{ij}(\mathbf{p})$ for all $(i,j)\in\mathcal{M}$ \hfill\eqref{eq:nearest_alias}, \eqref{eq:fixed_residual}
        \STATE Compute $r_a(\mathbf{p})$ for all anchors \hfill\eqref{eq:anchor_residual}
        \STATE Stack residuals into $\mathbf{r}(\mathbf{p})$ and Jacobian $\mathbf{J}_r$
        \STATE Solve~\eqref{eq:gauss_newton} for $\delta\mathbf{p}$; \ $\mathbf{p}\leftarrow\mathbf{p}+\delta\mathbf{p}$
        \IF{$\|\delta\mathbf{p}\|$ is negligible}
            \STATE \textbf{break}
        \ENDIF
    \ENDFOR
    \STATE $\mathbf{p}_k\leftarrow\mathbf{p}$, \ $J_k\leftarrow\|\mathbf{r}(\mathbf{p}_k)\|_2^2$
\ENDFOR
\RETURN $\hat{\mathbf{p}}_{\mathrm{init}}\leftarrow\mathbf{p}_{k^{\star}}$, \ $k^{\star}=\argmin_k J_k$
\end{algorithmic}
\end{algorithm}

\vspace{1mm}
\noindent
\textbf{Candidate scoring.}
For a distributed fixed-receiver pair $(i,j)$, we define the alias-aware residual
\begin{equation}
    e_{ij}(\mathbf{p})
    =
    \mathrm{modulo}_{\lambda}
    \left(
    \Delta\rho_{ij}(\mathbf{p})-\hat{z}_{ij}
    \right),
    \label{eq:alias_residual}
\end{equation}
which measures the distance from the predicted range difference to the nearest feasible alias branch in~\eqref{eq:alias_branch}, without requiring the alias index $k_{ij}$ to be known. Each curve sample $\mathbf{p}_n\in\mathcal{P}$ is assigned the score
\begin{equation}
    S(\mathbf{p}_n)
    =
    \sum_{(i,j)\in\mathcal{M}}
    q_{ij}
    \exp\left(
    -\frac{e_{ij}^{2}(\mathbf{p}_n)}
    {2\sigma_{\rho}^{2}}
    \right),
    \label{eq:candidate_score}
\end{equation}
where $\mathcal{M}$ is the set of available fixed-receiver pairs, $\sigma_\rho$ is the fixed-pair measurement uncertainty, and $q_{ij}$ is the pair reliability from Sec.~\ref{subsec:adaptive_weighting}. The bounded Gaussian kernel limits the influence of any single inconsistent pair, so a few corrupted measurements cannot dominate the score. Anchor residuals are excluded from~\eqref{eq:candidate_score} because all curve samples already satisfy the anchor constraints. Since a wrapped residual can be small on several alias branches simultaneously, the score may contain multiple local maxima along $\mathcal{C}$. We therefore keep the $K$ highest-scoring samples as candidates rather than committing to the single best sample.

\vspace{1mm}
\noindent
\textbf{Nearest-alias assignment.}
Each candidate is refined by a local weighted nonlinear least-squares optimization. To obtain a differentiable objective, each fixed-pair measurement is locally unwrapped onto the nearest alias branch 
$\mathbf{p}$:
\begin{equation}
    k_{ij}^{\star}(\mathbf{p})
    =
    \operatorname{round}
    \left(
    \frac{
    \Delta\rho_{ij}(\mathbf{p})
    -
    \hat{z}_{ij}
    }{\lambda}
    \right),
    \label{eq:nearest_alias}
\end{equation}
with the corresponding unwrapped target
\begin{equation}
    d_{ij}^{\star}(\mathbf{p})
    =
    \hat{z}_{ij}
    +
    k_{ij}^{\star}(\mathbf{p})\,\lambda.
    \label{eq:local_target}
\end{equation}
By construction, $\Delta\rho_{ij}(\mathbf{p})-d_{ij}^{\star}(\mathbf{p})$ equals the wrapped residual $e_{ij}(\mathbf{p})$ in~\eqref{eq:alias_residual}. The refinement therefore minimizes the same quantity used for scoring, and the alias index in~\eqref{eq:nearest_alias} is selected implicitly rather than enumerated over $\mathcal{K}_{ij}$.

\begin{figure*}[t]
    \centering
    \includegraphics[width=\linewidth]{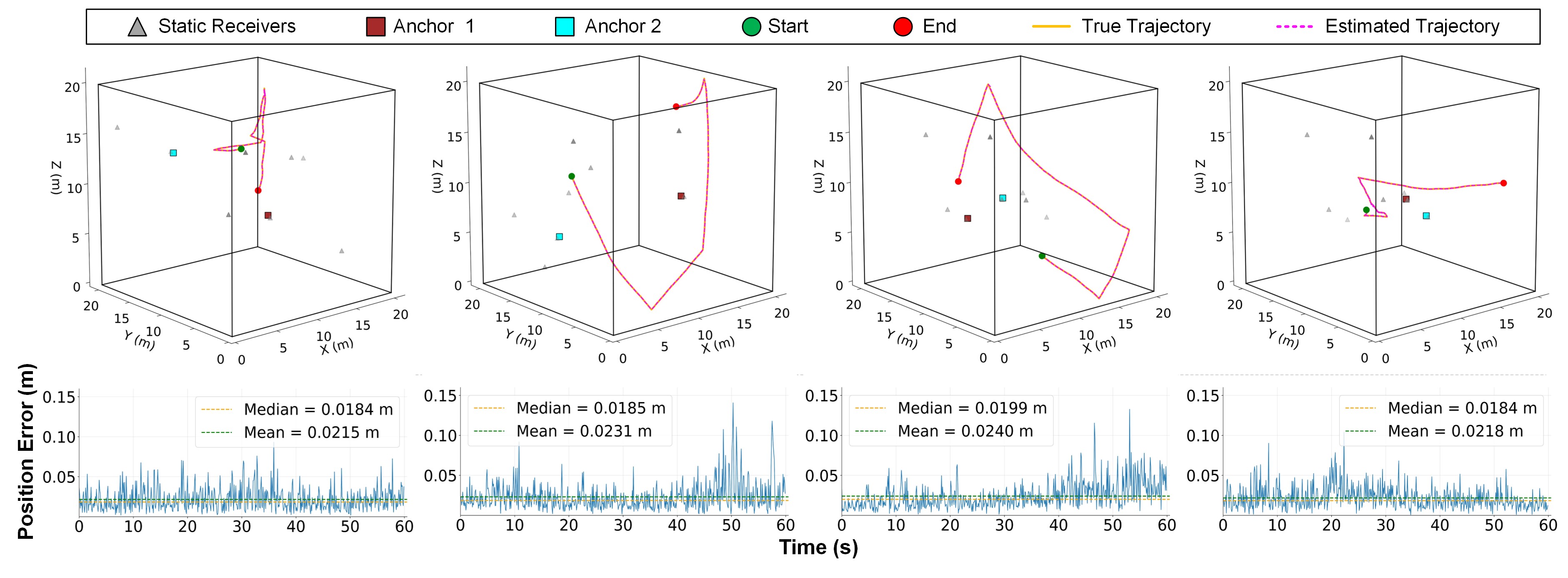}
    \vspace{-3mm}
    \caption{Monte Carlo acoustic simulation results are illustrated in a $20$\,m$\times$$20$\,m$\times$$20$\,m environment. The top row shows the true and estimated source trajectories for four randomized trials; the bottom row shows the corresponding tracking errors. More results are shown in Fig.~\ref{fig:sim_anchor_pairs} to demonstrate geometric robustness.}
    \vspace{-2mm}
    \label{fig:plots_3d}
\end{figure*}

\vspace{1mm}
\noindent
\textbf{Weighted Gauss-Newton refinement.}
For each fixed pair, the weighted residual is
\begin{equation}
    r_{ij}(\mathbf{p})
    =
    \frac{\sqrt{q_{ij}}}{\sigma_{\rho}}
    \left[
    \Delta\rho_{ij}(\mathbf{p})
    -
    d_{ij}^{\star}(\mathbf{p})
    \right],
    \label{eq:fixed_residual}
\end{equation}
and for each anchor pair $a$,
\begin{equation}
    r_a(\mathbf{p})
    =
    \frac{1}{\sigma_a}
    \left[
    \Delta\rho_a(\mathbf{p})
    -
    \hat{d}_a
    \right]
    =
    \frac{f_a(\mathbf{p})}{\sigma_a}.
    \label{eq:anchor_residual}
\end{equation}
Unlike scoring, the refinement includes the anchor residuals. This keeps the refined estimate close to $\mathcal{C}$ while allowing small deviations that absorb anchor-measurement noise. Stacking all fixed-pair and anchor residuals into $\mathbf{r}(\mathbf{p})$, each candidate solves
\begin{equation}
    \min_{\mathbf{p}}\;
    \frac{1}{2}
    \left\|
    \mathbf{r}(\mathbf{p})
    \right\|_2^2.
    \label{eq:refinement_objective}
\end{equation}
Since $k_{ij}^{\star}$ is piecewise constant in $\mathbf{p}$, the gradient of~\eqref{eq:fixed_residual} is simply $(\sqrt{q_{ij}}/\sigma_\rho)\,\nabla_{\mathbf{p}}\Delta\rho_{ij}(\mathbf{p})$, and the anchor rows follow from~\eqref{eq:anchor_gradient}. Let $\mathbf{J}_r$ denote the Jacobian of the stacked residual vector. The damped Gauss-Newton step solves
\begin{equation}
    \left(
    \mathbf{J}_r^{\top}\mathbf{J}_r
    +
    \gamma\mathbf{I}
    \right)
    \delta\mathbf{p}
    =
    -
    \mathbf{J}_r^{\top}\mathbf{r},
    \label{eq:gauss_newton}
\end{equation}
followed by the update
\begin{equation}
    \mathbf{p}
    \leftarrow
    \mathbf{p}
    +
    \delta\mathbf{p},
    \label{eq:gn_update}
\end{equation}
where $\gamma>0$ is a damping factor. The alias indices in~\eqref{eq:nearest_alias} are recomputed after every update, and the iteration runs for at most $I_{\max}$ steps or until $\|\delta\mathbf{p}\|$ becomes negligible.

After all $K$ candidates are refined, the one with the smallest final cost $\|\mathbf{r}(\mathbf{p})\|_2^2$ is selected as the initial source-position estimate $\hat{\mathbf{p}}_{\mathrm{init}}$. Algorithm~\ref{alg:curve_refinement} summarizes the complete Stage~1 procedure.
% the anchor constraints first reduce the search from the full monitored region to a sampled one-dimensional curve, the fixed-pair measurements then score the curve samples, and the top-$K$ candidates are refined with locally unwrapped measurements.

\begin{table*}[t]
    \centering
    \caption{Key configuration parameters used in the Monte Carlo simulations.}
    \vspace{-1mm}
    \label{tab:sim_config}
    \setlength{\tabcolsep}{5pt}
    \begin{tabular}{@{}l c l c@{}}
        \toprule
        \textbf{Parameter} & \textbf{Value}
        & \textbf{Parameter} & \textbf{Value} \\
        \midrule
        Number of trials
        & $100$
        & Environment size
        & $20\times20\times20$\,m$^3$ \\

        Static receivers
        & $8$
        & Source frequency
        & $1$\,kHz \\

        Anchor baseline
        & $0.475\lambda$
        & Sampling rate / frame
        & $44$\,kHz / $0.20$\,s \\

        Wrapped-pair noise, $\sigma_\rho$
        & $0.002$\,m
        & Anchor noise, $\sigma_a$
        & $0.001$\,m \\

        PF particles
        & $3000$
        & Position / velocity process noise
        & $0.3$\,m / $0.2$\,m/s \\

        Resampling threshold, $\beta$
        & $0.55$
        & Initialization seed samples
        & $10^6$ \\

        Curve-tracing budget
        & $800$ steps/direction
        & Top-$K$ refinement candidates
        & $10$ \\
        \bottomrule
    \end{tabular}
    \vspace{-2mm}
\end{table*}

\subsection{Stage 2: Temporal Tracking}
\label{sec:temporal_tracking}
After initial localization, the mobile receiver remains at the last anchor position ($a=A$) and continues to provide one live unaliased measurement $\hat{d}_{A,t}$ at each time step $t$. In this stage, a particle filter with $N_p$ particles estimates the source state $\mathbf{x}_t=[\mathbf{p}_t^\top,\mathbf{v}_t^\top]^\top$, allowing the source to move after initialization.

\vspace{1mm}
\noindent
\textbf{Initialization and motion model.}
Particle positions are drawn from a Gaussian cloud around the Stage~1 estimate, $\mathbf{p}_0^{(n)}\sim\mathcal{N}(\hat{\mathbf{p}}_{\mathrm{init}},\boldsymbol{\Sigma}_{p,0})$, while particle velocities are initialized near zero, $\mathbf{v}_0^{(n)}\sim\mathcal{N}(\mathbf{0},\boldsymbol{\Sigma}_{v,0})$. We adopt a constant-velocity transition model:
\begin{equation}
\begin{aligned}
    \mathbf{p}_t
    &=
    \mathbf{p}_{t-1}
    +
    \mathbf{v}_{t-1}\Delta t
    +
    \mathbf{w}_{p,t},\\
    \mathbf{v}_t
    &=
    \mathbf{v}_{t-1}
    +
    \mathbf{w}_{v,t},
\end{aligned}
\label{eq:pf_transition}
\end{equation}
where $\mathbf{w}_{p,t}$ and $\mathbf{w}_{v,t}$ are zero-mean Gaussian process noise terms.

\vspace{1mm}
\noindent
\textbf{Measurement likelihood.}
The measurement vector $\mathbf{z}_t$ contains the live anchor measurement $\hat{d}_{A,t}$ and the wrapped fixed-pair measurements $\{\hat{z}_{ij,t}\}_{(i,j)\in\mathcal{M}}$. Treating these measurements as conditionally independent given the source position, the likelihood of particle $n$ is
\begin{equation}
\begin{split}
    p\!\left(\mathbf{z}_t \mid \mathbf{x}_t^{(n)}\right)
    \propto{}&
    \exp\!\left(
    -\frac{\big(\hat{d}_{A,t}-\Delta\rho_A(\mathbf{p}_t^{(n)})\big)^{2}}{2\sigma_a^{2}}
    \right)\\
    &\times
    \prod_{(i,j)\in\mathcal{M}}
    \exp\!\left(
    -\frac{q_{ij,t}\,e_{ij}^{2}(\mathbf{p}_t^{(n)})}{2\sigma_\rho^{2}}
    \right),
\end{split}
\label{eq:pf_likelihood}
\end{equation}
where $e_{ij}(\cdot)$ is the alias-aware residual in~\eqref{eq:alias_residual} and $q_{ij,t}$ is the pair reliability from Sec.~\ref{subsec:adaptive_weighting}, recomputed from the current audio frame. The quality weight acts as a per-pair tempering exponent, so low-SNR pairs contribute flatter likelihood terms and have less influence on the particle weights. Because the wrapped residual is small on every feasible alias branch, the fixed pairs alone would admit multiple modes. The unaliased anchor term and the motion prior keep the particle cloud on the branch identified in Stage~1. Taking the logarithm of~\eqref{eq:pf_likelihood} recovers the refinement cost in~\eqref{eq:refinement_objective} restricted to the live anchor, so both stages evaluate measurements consistently.

\vspace{1mm}
\noindent
\textbf{Weight update and resampling.}
The particle weights are updated recursively as
\begin{equation}
    w_t^{(n)}
    \propto
    w_{t-1}^{(n)}
    \,
    p\!\left(
    \mathbf{z}_t
    \mid
    \mathbf{x}_t^{(n)}
    \right),
    \qquad
    \sum_{n=1}^{N_p} w_t^{(n)} = 1,
    \label{eq:pf_weight_update}
\end{equation}
and the state estimate is the weighted particle mean, $\hat{\mathbf{x}}_t=\sum_{n} w_t^{(n)}\mathbf{x}_t^{(n)}$. To prevent particle degeneracy, we compute the effective sample size $N_{\mathrm{eff}}=1/\sum_{n}\big(w_t^{(n)}\big)^2$ and apply systematic resampling whenever $N_{\mathrm{eff}}<\beta N_p$, where $\beta\in(0,1)$ is the resampling threshold.

\paragraph{Design Rationale.}
We considered several alternatives for resolving narrowband phase ambiguity. First, a dense three-dimensional grid search is attempted, which is globally exhaustive but computationally expensive and resolution-dependent. At the same time, a globally initialized particle filter requires an impractically large particle population to represent the many aliased modes. Second, graph-based phase-consistency methods can enforce compatibility among pairwise measurements, but they do not yield an efficient spatial search. Moreover, a single anchor reduces the feasible set only to a two-dimensional surface in 3D and therefore leaves an expensive manifold search. The adopted design instead uses two independent anchor constraints to reduce the initial search to a one-dimensional intersection curve, applies the wrapped fixed-pair measurements only to score and refine candidates on that curve, and initializes temporal filtering after the ambiguity has been resolved.

\section{Evaluation}

We evaluate the proposed localization and tracking framework through Monte Carlo acoustic simulations, ROS-Gazebo robotic simulations, and open-water field experiments. First, the acoustic evaluation examines localization and tracking performance across randomized receiver geometries, noise levels, and source motion. Next, we integrate the algorithm with mobile robots in floating and submerged infrastructure surveillance scenarios in the Gazebo\texttrademark\,simulator. We then present real-world results from lake trials, highlighting the practicality of the approach under noisy sensing conditions. Finally, we analyze sensitivity to measurement uncertainty and quantify the benefit of quality-adaptive pair weighting.

We use three performance metrics. The \emph{initial localization error} (ILE) is the Euclidean distance between the estimated and true source positions after completing Stage 1. The \emph{success rate} (SR) is the percentage of trials satisfying $\mathrm{ILE}<1\,\mathrm{m}$. During Stage 2, the \emph{tracking error} (TE) is the time-varying Euclidean distance between the estimated and true source positions. Note that TE is reported only for trials with successful initialization.

\begin{table}[h]
    \caption {The effect of the number of anchor constraints on localization and tracking performance is presented. ILE is reported over all trials as mean$\pm$std, whereas TE is reported only for successful trials satisfying $\mathrm{ILE}<1\,\mathrm{m}$.}
    \centering
    \small
    \begin{tabular}{c||c|c|c}
    \hline
    Approach      & SR (\%)    & ILE (m) & TE (m)  \\
    \hline
    No anchor   & $4\%$ & $11.96\pm4.85$ & $0.05\pm0.02$    \\
    One anchor  & $12\%$ & $9.92\pm5.59$ & $0.03\pm0.07$     \\
    \textbf{Two anchors} & $91\%$ & $0.72\pm2.01$ &  $0.02\pm0.03$ \\
    \hline
    \end{tabular}
    % \vspace{-2mm}
    \label{tab:baseline}
\end{table}

\begin{figure*}[p]
    \centering
    \includegraphics[width=0.9\linewidth]{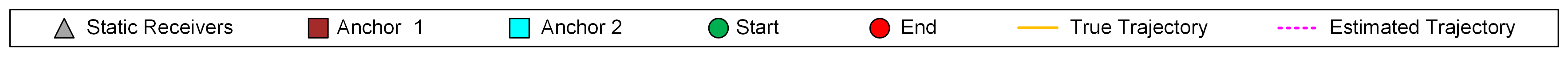}
    % \vspace{0.5cm}
    \includegraphics[width=0.92\linewidth]{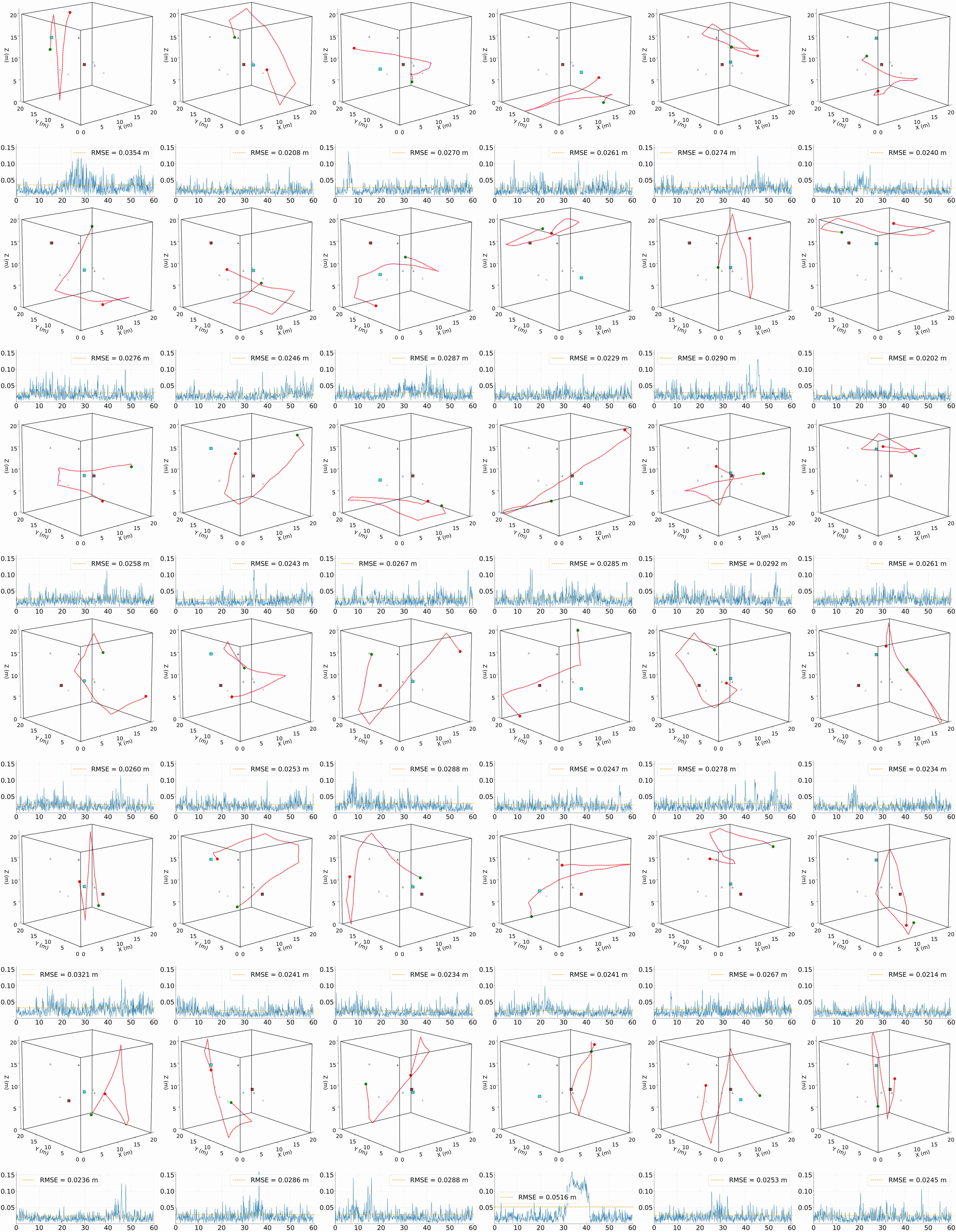}
    \caption{Localization and tracking results are illustrated for all $36$
    anchor-pair choices. Static-receiver locations, source initialization,
    and source-motion direction are randomized in each trial to demonstrate
    geometric robustness.}
    \label{fig:sim_anchor_pairs}
\end{figure*}

\begin{figure*}[t]
    \centering
    \includegraphics[width=\linewidth]{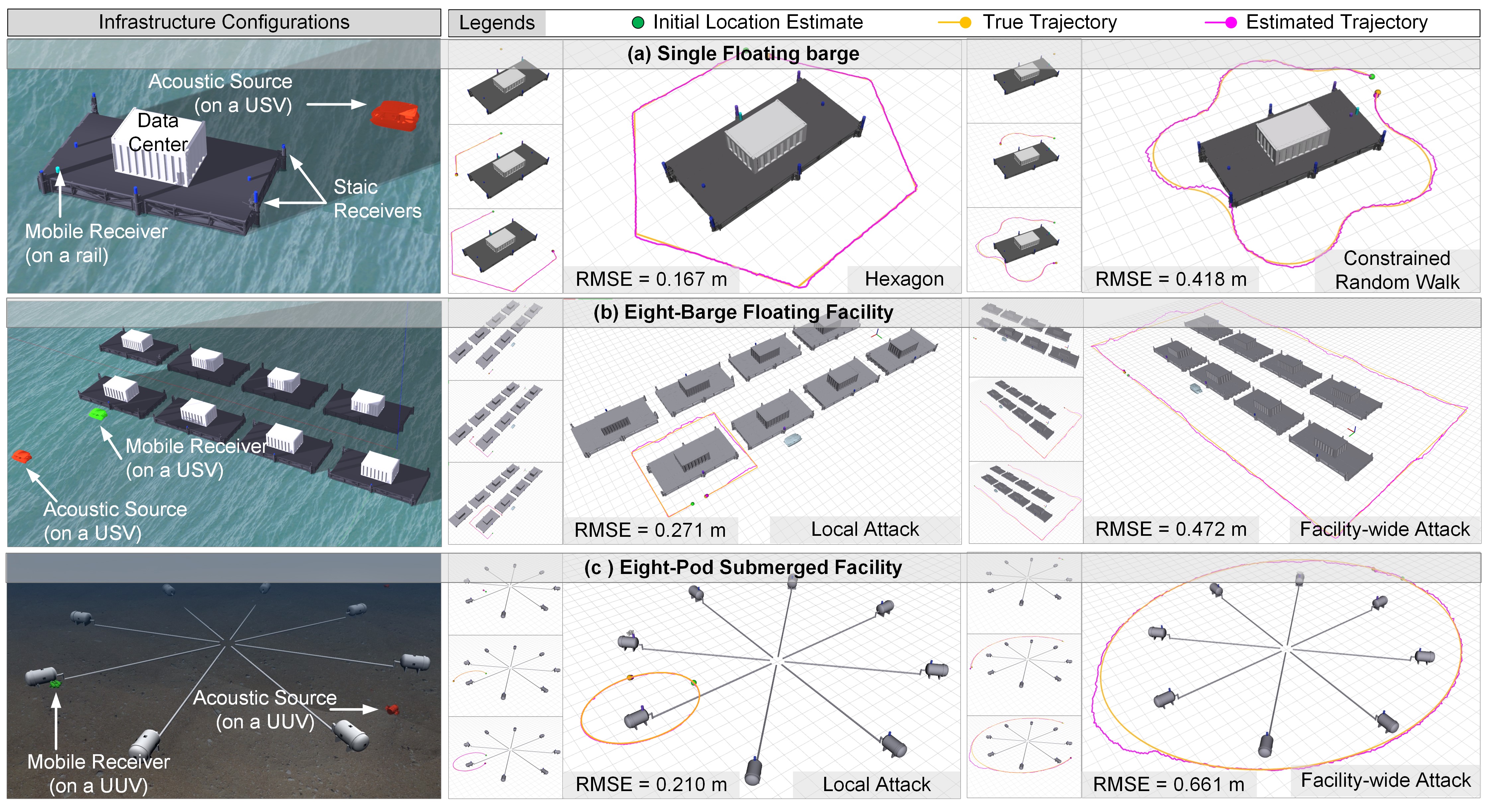}
    \vspace{-3mm}
    \caption{Three infrastructure surveillance scenarios are demonstrated: (a) single floating barge, (b) eight-barge facility, and eight-pod submerged facility. For the single-barge scenario, the mobile receiver moves along a rail while a USV-mounted acoustic source follows hexagonal and random walk trajectories. In the distributed floating and submerged scenarios, the mobile receiver is carried by a USV or UUV, respectively, and the source executes both local attacks around one unit or facility-wide attacks. Three intermediate snapshots during tracking are shown on the left of each result.}
    \vspace{-3mm}
    \label{fig:gazebo}
\end{figure*}

\subsection{Monte-Carlo Simulation}
We conduct acoustic simulations in a $20 \times 20 \times 20\,\mathrm{m}^3$ environment using \texttt{Pyroomacoustics}. Each Monte Carlo trial randomizes the locations of the eight static receivers, the selected anchor receivers, the initial source position, and the source-motion direction. The source remains stationary while the mobile receiver acquires the anchor measurements (Stage 1) and then follows a random 3D trajectory (Stage 2). This randomized setup evaluates robustness to receiver geometry and source motion. Fig.~\ref{fig:plots_3d} shows representative results from four randomly selected trials, highlighting stable tracking across diverse receiver placements and trajectory shapes. The corresponding tracking-error plots demonstrate $<5$\,cm steady-state error with limited drift. Table~\ref{tab:sim_config} summarizes the key simulation settings for reproducibility.

We next examine the role of the anchor constraints by comparing the proposed two-anchor configuration with one-anchor and no-anchor baselines. As explained in Sec.~\ref{subsec:anchor_constraint}, one anchor constrains the source only to a 2D surface, while two independent anchor constraints reduce the candidate set to a 1D intersection curve. Since multiple ambiguous solutions exist for fewer than two anchors, this leads to poor performance in the one- and no-anchor cases. As summarized in Table~\ref{tab:baseline}, the no-anchor and one-anchor baselines achieve SRs of only $4\%$ and $12\%$, respectively, since the solution space is under-constrained. In contrast, the proposed two-anchor configuration achieves an SR of $91\%$ and reduces the mean ILE to $0.72$\,m. Notably, 
this comparison answers RQ1, demonstrating that two independent anchor constraints substantially improve initialization reliability under narrowband phase ambiguity.

% the tracking error remains consistently low across all approaches ($<10$\,cm) for successful trials, indicating that resolving the initial phase-aliased solution space is the principal challenge.

\vspace{1mm}
\noindent
\textbf{Robustness to Anchor Selection.}
To isolate the effect of anchor choice, we fix a seven-receiver geometry and repeat localization and tracking for all $36$ possible ordered anchor pairs. The source initialization and motion direction are randomized as in the preceding trials. Fig.~\ref{fig:sim_anchor_pairs} shows that, after successful initialization, the estimated trajectories remain close to the ground truth and the tracking errors remain small across the tested anchor configurations. The result indicates that the method is not tied to one favorable anchor pair, although initialization quality can still vary with anchor geometry and measurement quality.

\subsection{Robotic Simulation in Gazebo-ROS}
We integrate the proposed algorithm with mobile robots in ROS1-Gazebo and generate acoustic measurements online from the instantaneous source and receiver poses. Fig.~\ref{fig:gazebo} shows three infrastructure configurations: a single floating barge with a rail-mounted mobile receiver, an eight-barge floating facility spanning approximately $60\times25\,\mathrm{m}^{2}$, and an eight-pod submerged facility arranged over a $40\times40\,\mathrm{m}^{2}$ area. The source and the mobile receiver are carried by USVs or UUVs in the floating and submerged cases, respectively.

\begin{figure*}[t]
    \centering
    \includegraphics[width=\linewidth]{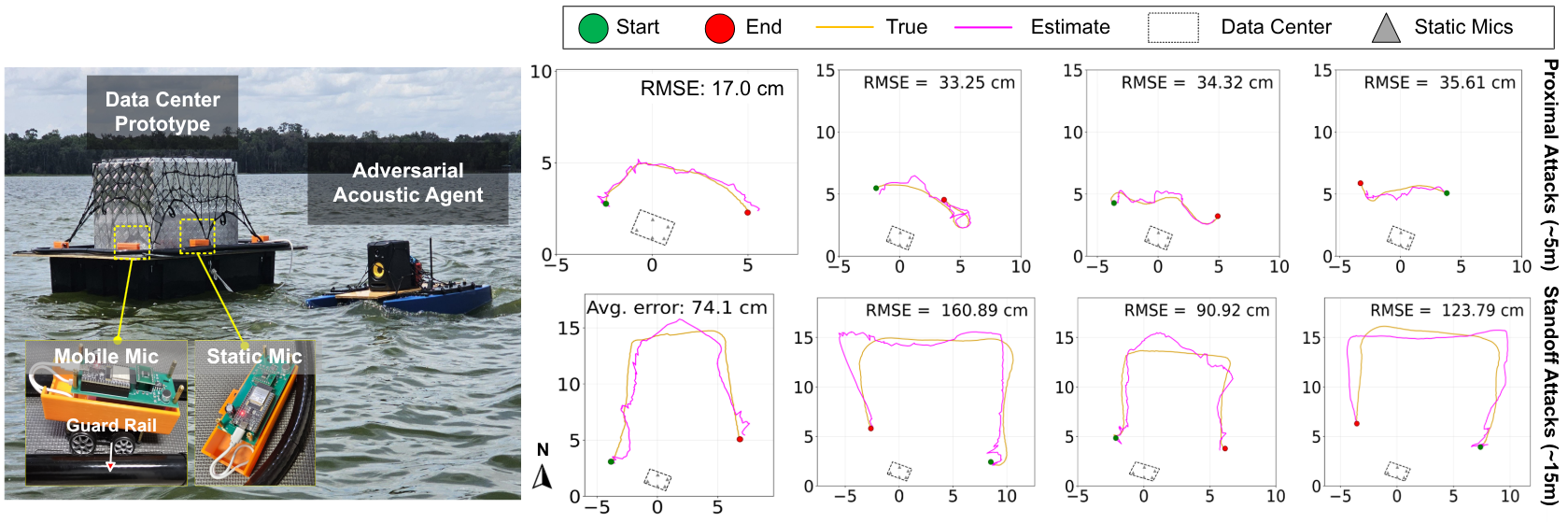}
    \vspace{-2mm}
    \caption{The lake experimental setup includes a scaled floating data center, distributed static microphones, a mobile microphone on a four-wheeled vehicle, and a USV carrying the acoustic source; the top row shows four proximal attacks executed by the USV, and the bottom row shows standoff attacks. The ground-truth path is obtained from the USV's GPS and compared against the estimated trajectory.}
    \vspace{-3mm}
    \label{fig:field}
\end{figure*}

During an acoustic attack, the receiver with the highest signal strength is selected as the first anchor, which typically corresponds to the unit closest to the source. A neighboring unit is selected as the second anchor, and the mobile robot remains there during stage 2. The duration of stage 1 depends on the facility scale and receiver spacing. For example, in the multi-barge scenario, a surveillance USV traveling at a nominal speed of  $1\,\mathrm{m/s}$ reaches two consecutive anchor locations within $10$\,s. This interval is compatible with the considered threat model, in which the source remains near one target long enough to maximize damage.

The six evaluated trajectories include localized attacks around one unit or facility-wide motion around the full infrastructure. The corresponding tracking RMSEs remain below $1$\,m in all cases. We observe lower TE for localized attacks, since nearby receivers provide stronger signals and better geometry. During wide trajectories, larger deviations occur in specific regions with weaker receiver coverage or less informative geometry. This highlights that receiver placement should be designed jointly with the expected patrol region to maintain signal strength and geometric observability over the whole facility.

\subsection{Open-Water Field Validation}
\subsubsection{Receiver Hardware and Data Acquisition}

We finally validate the proposed framework in a shallow-water lake using a $4'\times6'$ floating data center prototype equipped with six distributed receiver modules; see Fig.~\ref{fig:field}. Each sensing module is built on a custom-designed PCB and consists of an electret microphone, a two-stage LT1632CN8 operational amplifier, an ADS8699 SAR analog-to-digital converter, and an ESP32 development board for audio data acquisition and wireless transmission. The sampling rate is set to $62$\,kHz to enable high-resolution TDOA analysis. Each module transmits digitized audio data over WiFi to a central computer for subsequent estimation.

Fig.~\ref{fig:receiver_hardware} shows the analog front end and the assembled wireless receiver. The two-stage amplifier conditions the microphone signal before digitization, while the PCB integrates signal acquisition and wireless transmission into a deployable sensing module.

\begin{figure}[h]
    \centering
    \begin{subfigure}[b]{0.8\linewidth}
        \centering
        \includegraphics[width=\linewidth]
        {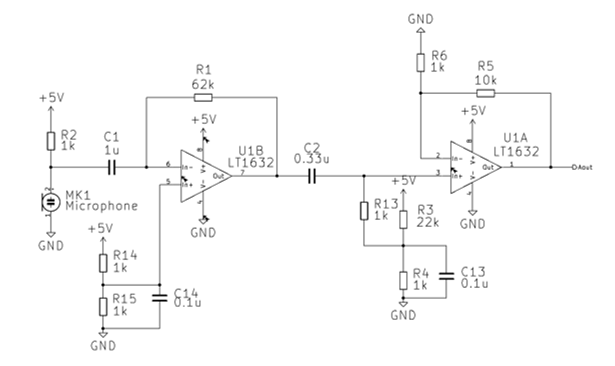}
        \caption{Analog front-end circuit.}
    \end{subfigure}\hfill
    \begin{subfigure}[b]{0.17\linewidth}
        \centering
        \includegraphics[width=\linewidth]
        {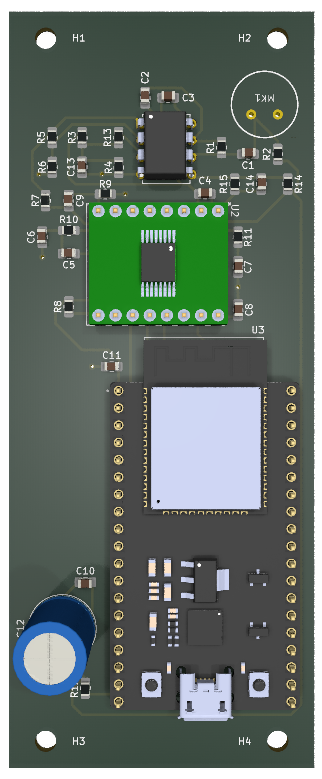}
        \caption{PCB.}
    \end{subfigure}
    \caption{The custom receiver modules enable distributed acoustic signal acquisition and wireless transmission to a central computer for AASL.}
    \label{fig:receiver_hardware}
\end{figure}

\subsubsection{Field Setup and Results}
In our floating data center scenario on a lake, the mobile receiver module is mounted on a small remote-controlled vehicle that runs between guard rails and approaches selected static receivers to form the required anchor pairs. The adversarial acoustic source is realized using a KRK Classic5 speaker mounted on a BlueBoat\texttrademark\,USV. The USV's motion, acoustic emission frequency, and emission amplitude are remotely controlled by a human operator.

Fig.~\ref{fig:field} shows the lake experiment setup and $8$ results selected from $10$ trials: four proximal acoustic attacks conducted within $5$\,m of the setup, and four standoff attacks from distances up to $15$\,m. In $9$ of the $10$ trials, our method successfully locates the source and maintains tracking error below $2$\,m throughout the tests. The tracking errors for the standoff cases are typically higher due to increased signal attenuation and weaker geometry, but remain below $1$\,m in the representative trial shown. These results demonstrate that the proposed framework remains effective under real sensing, communication, and synchronization imperfections.

\subsection{Sensitivity to Measurement Uncertainty}
\label{sec:noise_sensitivity}
TDOA-based localization is sensitive to uncertainty in receiver geometry, as position errors perturb the estimated range differences. During anchor formation, the mobile receiver must remain at a known short-baseline distance relative to a static receiver; however, platform motion and self-positioning error in a dynamic water environment introduce uncertainty into this geometry. We model the combined effect of mobile-receiver position error and anchor-measurement noise as an equivalent anchor range-difference uncertainty, $\sigma_a$. To characterize its impact, we vary $\sigma_a$ together with the fixed-pair measurement uncertainty, $\sigma_\rho$. For each $(\sigma_a,\sigma_\rho)$ combination, we perform $100$ Monte Carlo trials with randomized receiver geometries and source locations, and report SR and ILE.

\begin{figure}[h]
    \centering
    \includegraphics[width=\linewidth]{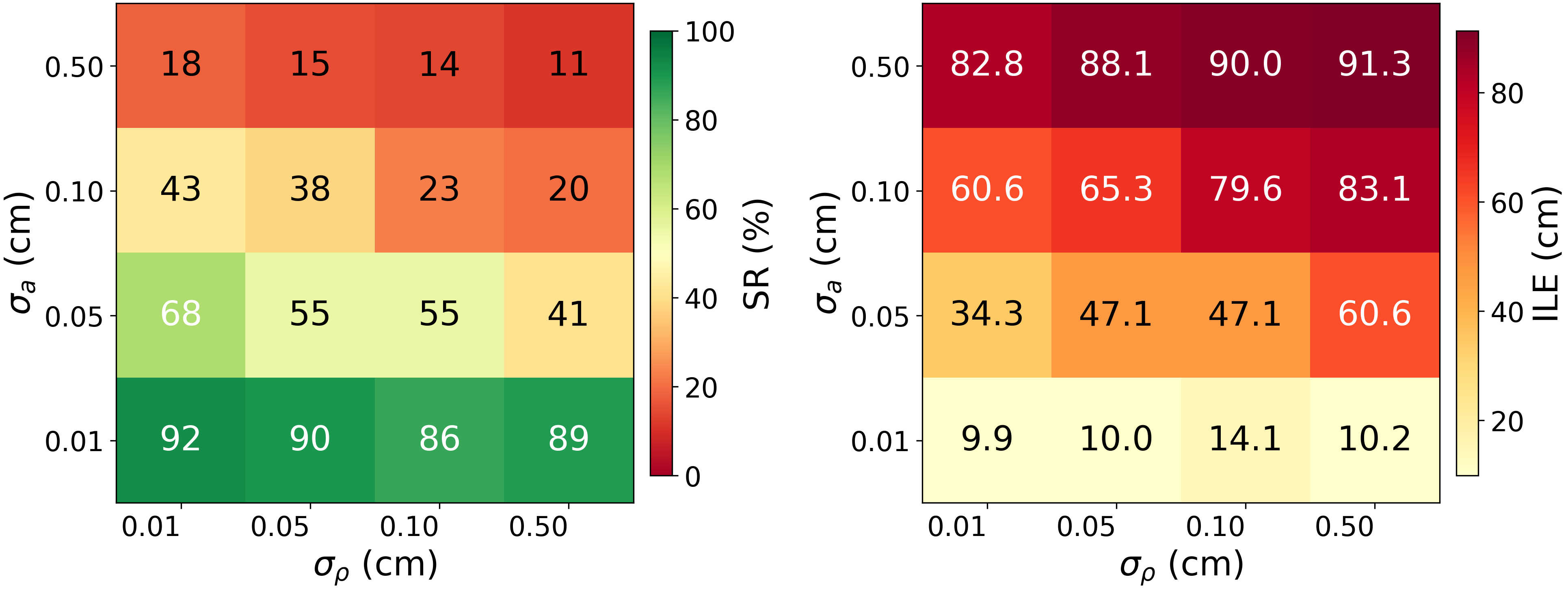}
    \vspace{-2mm}
    \caption{Sensitivity of the initial localization (Stage 1) to anchor-pair and fixed-pair measurement uncertainty is analyzed in terms of success rate (SR) and initial localization error (ILE). Each cell is averaged over $100$ Monte Carlo trials. The results show that performance is dominated by the anchor measurement noise $\sigma_a$, while the effect of fixed-pair noise $\sigma_\rho$ is comparatively moderate.}
    \vspace{-2mm}
    \label{fig:sigma_sweep}
\end{figure}

Fig.~\ref{fig:sigma_sweep} shows that the localization performance is more sensitive to the anchor measurements than to the fixed-pair measurements. At low anchor-pair uncertainty ($\sigma_a=0.01$\,cm), the success rate remains high at $86$-$92\%$ across all values of $\sigma_\rho$, and the corresponding ILE over successful trials remains near $10$-$14$\,cm. As $\sigma_a$ increases, the initialization performance degrades rapidly.

In contrast, increasing $\sigma_\rho$ for a fixed value of $\sigma_a$ produces gradual changes in SR and ILE. This behavior is consistent with the role of the two measurement types: the anchor measurements define the unaliased candidate manifold, whereas the fixed-pair measurements mainly disambiguate and score candidate points along that manifold. The key insight is that a poorly conditioned manifold with high $\sigma_a$ cannot be reliably traced regardless of the quality of the fixed-pair measurements. Empirically, we observe reliable localization when $\sigma_a \lesssim 0.05\,\text{cm}$. These findings address RQ2 by identifying anchor-measurement accuracy as a dominant factor for localization reliability.

% These results indicate that practical system design should prioritize accurate anchor-pair measurement and mobile-receiver positioning.

\subsection{Ablation: Adaptive Receiver Weighting}
\label{sec:adaptive_weighting_eval}

Distributed acoustic sensing is sensitive to channel degradation: a single noisy receiver can corrupt multiple TDOA pairs and affect localization. We therefore ablate the proposed quality-adaptive weighting by comparing it against a uniform weighting baseline. In each trial, one non-anchor static receiver is selected randomly, and its acoustic waveform is corrupted with additive white Gaussian noise. This corruption affects every TDOA pair containing that channel. We vary the corrupted-channel SNR from $-30$\,dB to $0$\,dB and include a clean condition. The uniform and adaptive estimators are tested on identical receiver geometries and source trajectories. The uniform estimator assigns $q_{ij}=1$ to every fixed-receiver pair. The adaptive technique estimates the received signal quality using~\eqref{eq:receiver_quality} and propagates it to the pair weights $q_{ij}$. 

\begin{figure}[h]
    \centering
    \includegraphics[width=\linewidth]{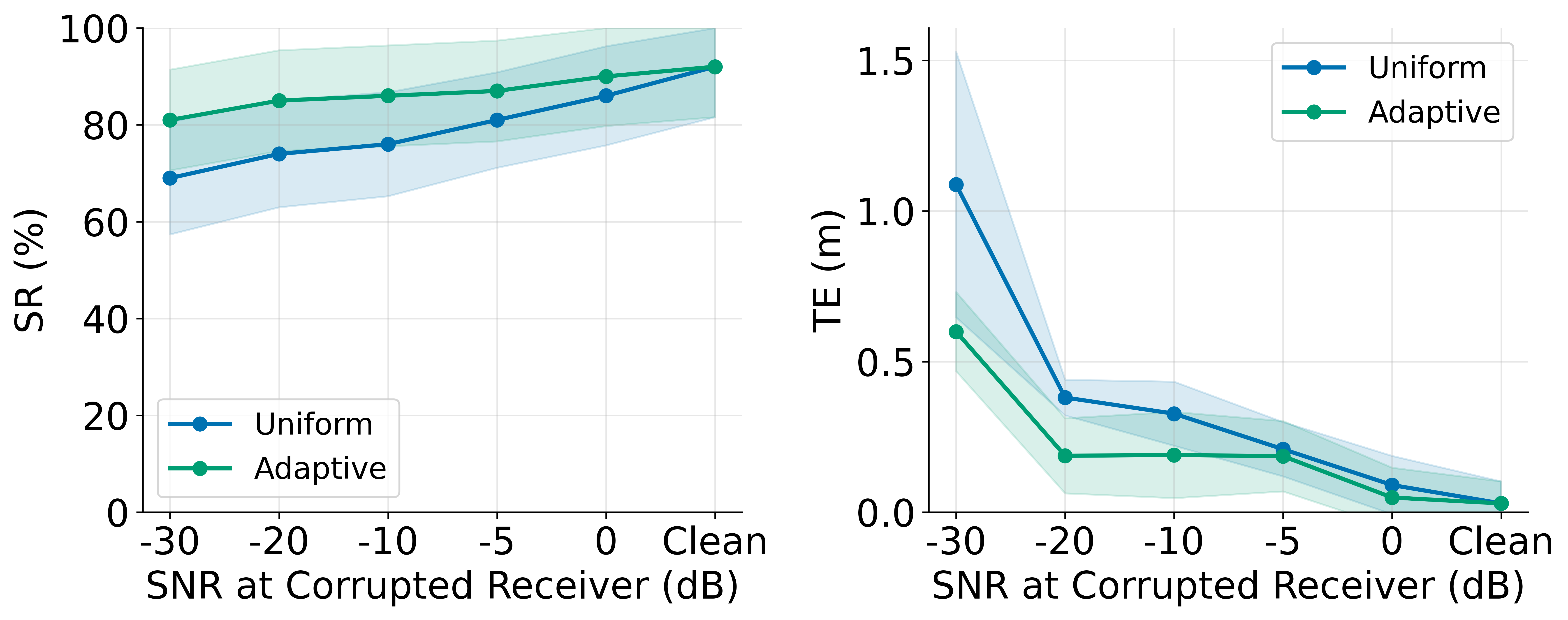}
    \vspace{-1mm}
    \caption{Effect of adaptive measurement weighting is analyzed under single-channel degradation. One randomly selected non-anchor receiver is corrupted at different SNR levels. SR and TE are reported; shaded regions denote $95\%$ confidence intervals. Adaptive weighting provides the largest improvement at low SNR and converges to uniform weighting under clean measurements.}
    \vspace{-1mm}
    \label{fig:adaptive_rx}
\end{figure}

Fig.~\ref{fig:adaptive_rx} shows that adaptive weighting provides noticeable benefit under severe channel degradation. At $-30$\,dB, the proposed weighting increases the success rate by approximately $10\%$ and reduces the mean tracking error from $1.09$\,m to $0.60$\,m. The performance gap decreases as the corrupted-channel SNR improves. Under clean measurements, both approaches converge to $92\%$ SR and $0.03$\,m TE.
This ablation addresses RQ3, showing that receiver-level quality weighting limits the influence of a locally corrupted channel without discarding the remaining useful information.

\subsection{Practical Sensing Challenges}
\label{sec:practical_challenges}

Our experiments reveal practical difficulties in obtaining reliable phase-based TDOA measurements. In particular, independent receiver clocks, intermittent wireless data loss, and multipath propagation introduce errors that are not captured by the measurement models. We summarize these observations to identify the key challenges and plan improvements for future deployments.

\subsubsection{Synchronization and Data Loss}

Each wireless sensing module operates with an independent sampling clock and transmits acoustic data over WiFi to a central computer. This architecture introduces challenges associated with sampling-rate mismatch, relative clock drift, variable communication latency, and packet loss. In earlier experiments targeting $48$\,kHz sampling, the receiver sampling rates initially differed by several hundred Hertz. Initial calibration reduced these discrepancies to tens of Hertz but did not eliminate the relative drift.

We investigated periodic reference-clock synchronization and timestamp-based packet alignment. However, the microcontrollers occasionally failed to respond to synchronization requests while handling continuous sampling and wireless transmission. Missing or corrupted packets also required discarding portions of the recorded data to maintain channel alignment. As illustrated in Fig.~\ref{fig:daq_challenge}, these interruptions resulted in fragmented multichannel recordings, while residual temporal misalignment produced inconsistencies in the relative phases of simultaneously recorded signals.

\begin{figure}[t]
    \centering
    \includegraphics[width=\linewidth]{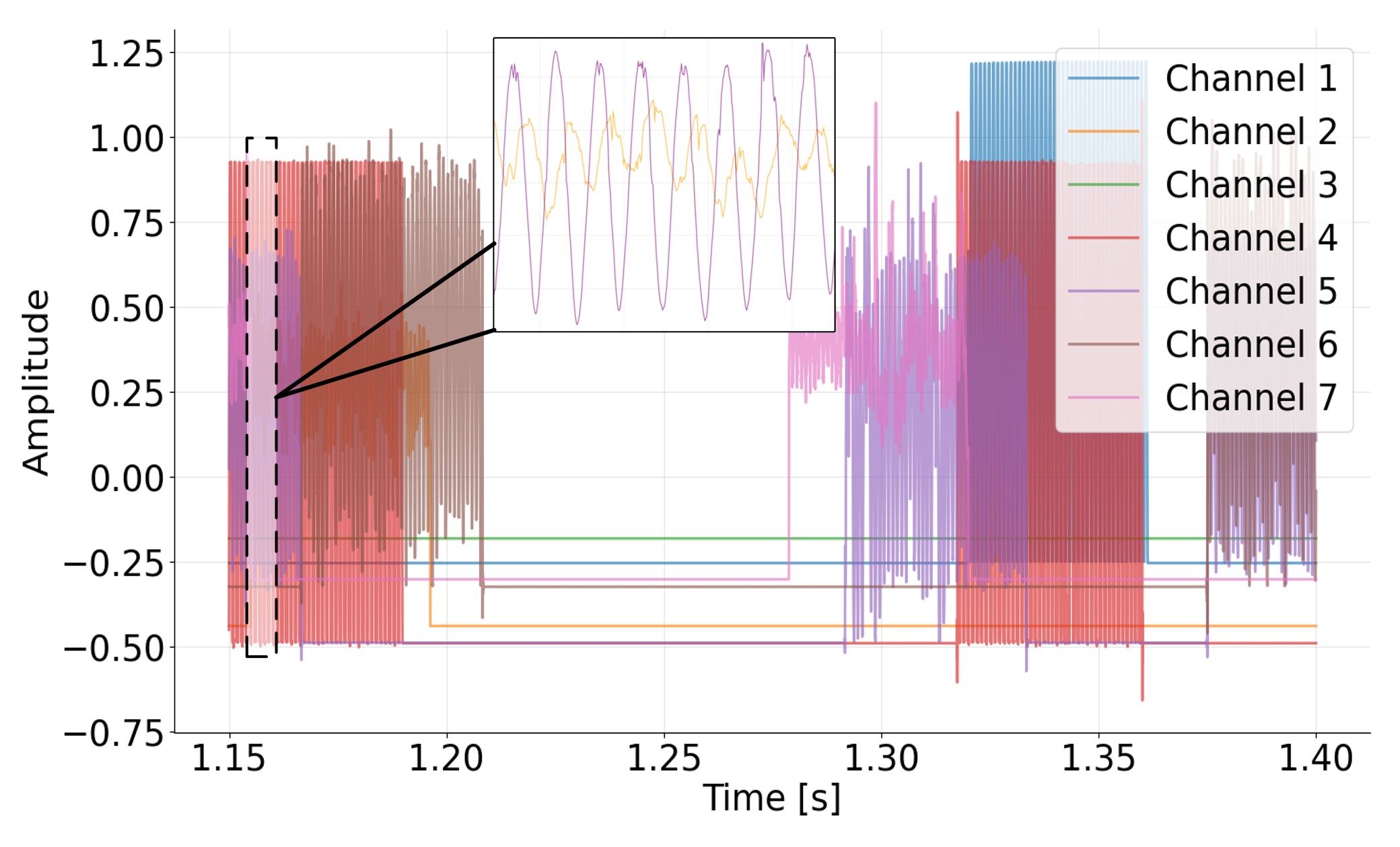}
    \vspace{-3mm}
    \caption{Synchronization and data-loss challenges observed during multichannel wireless acquisition. Flat segments indicate missing or discarded waveform data, while the inset illustrates phase misalignment between receiver channels. These effects compromise the temporal consistency required for phase-based TDOA estimation.}
    \label{fig:daq_challenge}
    \vspace{-2mm}
\end{figure}

These observations highlight that coarse packet alignment and initial sampling-rate calibration are insufficient to guarantee long-term phase coherence among independently clocked receivers. Reliable TDOA measurements require either synchronized acquisition hardware or continuous estimation and compensation of inter-receiver clock offsets and drift.

\subsubsection{Multipath-Induced Phase Instability}

Acoustic reflections from surrounding structures and environmental boundaries introduce additional uncertainty in narrowband phase measurements. The direct and reflected components combine at each receiver, potentially altering the observed signal phase and the corresponding TDOA estimate.

Fig.~\ref{fig:tdoa_mismatch} compares the measured and geometrically expected TDOA for a stationary source-receiver configuration over a one-second observation interval. Although the expected TDOA remains constant, the measured values exhibit substantial temporal fluctuations. These deviations are consistent with the combined effects of multipath interference, residual synchronization errors, and intermittent data loss. Consequently, the measured phase difference does not always provide a reliable estimate of the direct-path range difference.

\begin{figure}[t]
    \centering
    \includegraphics[width=\linewidth]{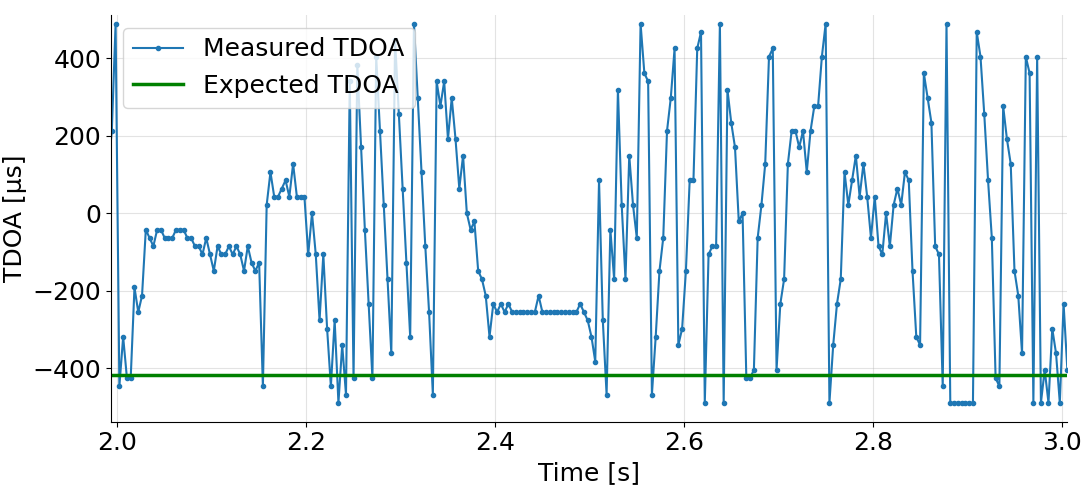}
    \vspace{-1mm}
    \caption{Measured and geometrically expected TDOA over a one-second interval with stationary source and receivers. Despite the fixed geometry, the phase-derived TDOA fluctuates substantially. Multipath interference, residual clock misalignment, and data loss may all contribute to the observed deviations.}
    \label{fig:tdoa_mismatch}
    \vspace{-2mm}
\end{figure}

To examine multipath effects separately, our earlier simulations investigated a $20\times20\times20$\,m environment with a $1$\,kHz tonal source and varying reflection orders and wall absorption coefficients. The results showed increasing phase deviations under stronger reflections, with mean errors exceeding $40^\circ$ even for first-order reflections at low absorption. At $1$\,kHz in air, such a phase deviation corresponds to approximately $3.8$\,cm of equivalent range-difference error. These findings indicate that multipath-induced phase distortion can substantially affect narrowband localization, even when receiver geometry is accurately known.

\noindent\textbf{Lessons for future efforts.}
The combined findings suggest several priorities for practical deployments: \textbf{(i)} improving inter-receiver synchronization through shared-clock acquisition or continuous clock-offset estimation; \textbf{(ii)} detecting packet loss and excluding temporally inconsistent measurements; \textbf{(iii)} calibrating receiver phase offsets and accurately estimating mobile-receiver position; and \textbf{(iv)} incorporating multipath-aware measurement models and reliability assessment. Alternative sensing architectures, including compact multi-element arrays for DOA estimation or matched-field processing, may also improve robustness and warrant further investigation.

\section{Conclusion}
\label{sec:conclusion}

This report investigated the feasibility of localizing and tracking narrowband adversarial acoustic sources using sparsely distributed receivers augmented with a mobile sensing platform. We developed a two-stage framework that exploits short-baseline anchor constraints to resolve phase ambiguity during initialization, followed by particle-filter tracking with quality-adaptive measurement weighting. Our analyses showed that two anchor constraints are essential for reliable 3D localization and that performance is more sensitive to anchor-pair measurement accuracy than fixed-pair measurements. Moreover, quality-adaptive weighting improved robustness when signal quality at individual receiver channels is degraded. Integration with ROS-Gazebo demonstrated scalable surveillance operation across floating and submerged facilities, while open-water experiments validated the complete system with sub-meter tracking error. Beyond these successful demonstrations, our experiments revealed practical limitations associated with independently synchronized receivers, wireless data loss, and multipath-induced phase instability. These observations emphasize that reliable phase-based localization depends not only on geometric observability but also on the quality and temporal consistency of the acquired measurements. Future efforts will prioritize improved receiver synchronization and phase calibration, reliable multichannel acquisition, accurate mobile-receiver positioning, and multipath-aware measurement processing.

% These results establish mobile-anchor-assisted sensing as a practical approach for acoustic surveillance of distributed offshore and underwater infrastructure. Future work will investigate surveillance policies, improved mobile-receiver pose estimation, and robustness to stronger multipath effects and intermittent adversarial emissions.

\section*{Acknowledgments}
This research is supported by the Office of Naval Research
(ONR) grant \#N000142412596.

% ---------- References ----------
{
\bibliographystyle{ieeetr}
\bibliography{bibs/combined}
}

% Supplementary 
% \clearpage
% \input{suppl}

\end{document}